\documentclass{article}
\usepackage[]{log_2026}						

\usepackage{booktabs}		
\usepackage{rotating}
\usepackage{multirow}						
\usepackage{amsfonts}						
\usepackage{graphicx}						
\usepackage{duckuments}						

\usepackage[numbers,compress,sort]{natbib}	
\usepackage[utf8]{inputenc}
\usepackage[T1]{fontenc}
\usepackage{amsmath,amssymb,amsthm,mathtools,bm}

\usepackage[most]{tcolorbox}
\usepackage{xcolor}
\usepackage{enumitem}

\definecolor{contribblue}{HTML}{2563EB}
\definecolor{contribbg}{HTML}{F4F8FF}

\newtcolorbox{contributionbox}{
  enhanced,
  breakable,
  colback=contribbg,
  colframe=contribblue!80!black,
  boxrule=0.7pt,
  arc=2mm,
  left=1.5mm,
  right=1.5mm,
  top=1.2mm,
  bottom=1.2mm,
  borderline west={2.5pt}{0pt}{contribblue},
  title=\textbf{Contributions},
  coltitle=black,
  fonttitle=\bfseries,
  attach boxed title to top left={xshift=2mm,yshift=-2mm},
  boxed title style={
    colback=white,
    colframe=contribblue!80!black,
    boxrule=0.7pt,
    arc=1.5mm,
    left=1mm,
    right=1mm,
    top=0.5mm,
    bottom=0.5mm
  }
}
\usepackage{algorithm}
\usepackage{algpseudocode}
\usepackage{booktabs}
\usepackage{multirow}
\usepackage{graphicx}
\usepackage{xcolor}
\usepackage{tikz}
\usetikzlibrary{arrows.meta,positioning,fit,backgrounds}
\usepackage{booktabs}      
\usepackage{multirow}      
\usepackage{adjustbox}     
\usepackage{array}         
\usepackage{makecell}      

\usepackage{amsmath}
\usepackage{amssymb}
\usepackage{bm}

\usepackage{graphicx}
\usepackage{subcaption}
\usepackage{xcolor}

\usepackage{siunitx}       
\usepackage{placeins}      
\newcommand{\R}{\mathbb{R}}
\newcommand{\E}{\mathbb{E}}
\newcommand{\N}{\mathbb{N}}
\newcommand{\Spd}{\mathbb{S}^{p}_{++}}
\newcommand{\Sym}{\mathbb{S}^{p}}
\newcommand{\Mtraj}{\mathcal{M}_{T}}
\newcommand{\Mtrajfut}{\mathcal{M}_{T-K}}
\newcommand{\dLE}{d_{\mathrm{LE}}}
\newcommand{\expmap}{\mathrm{Exp}}
\newcommand{\Ltraj}{\mathcal{L}_{\mathrm{traj}}}

\newcommand{\Ef}{E_{f}}

\newcommand{\tildePhi}{\tilde{\Phi}}
\newcommand{\tildePhif}{\tilde{\Phi}_{f}}
\newcommand{\mf}{m^{\mathrm f}}

\newcommand{\qrw}{p_{\mathrm{RW}}}

\newcommand{\hatTheta}{\hat{\Theta}}
\newcommand{\tvglcfm}{\textsc{TVGL--CFM}}
\DeclareMathOperator{\tr}{tr}
\DeclareMathOperator{\diag}{diag}

\DeclareMathOperator{\logm}{log}
\DeclareMathOperator{\expm}{exp}
\DeclareMathOperator{\veclt}{vec_{lt}}
\newcommand{\Th}{\Theta}
\newcommand{\norm}[1]{\left\lVert #1 \right\rVert}

\newcommand{\softthr}{\mathcal{S}}

\theoremstyle{plain}
\newtheorem{proposition}{Proposition}
\theoremstyle{remark}
\newtheorem{remark}{Remark}

\usepackage{amsfonts}						
\usepackage{float}							

\usepackage[utf8]{inputenc} 
\usepackage[T1]{fontenc}    
\usepackage{url}            
\usepackage{booktabs}       
\usepackage{amsfonts}       
\usepackage{nicefrac}       
\usepackage{microtype}      
\usepackage{xcolor}         

\title{\textbf{TVGL--CFM: Generating and Forecasting Time-Varying\\ Trajectories of Dynamic Networks\\
with Conditional Flow Matching}}

\author{%
  Om Roy\thanks{Corresponding author.} \\
  University of Strathclyde \\
  \texttt{o.roy.2022@uni.strath.ac.uk}
  \And
  Yashar Moshfeghi \\
  University of Strathclyde \\
  \texttt{yashar.moshfeghi@strath.ac.uk}
  \And
  Keith Malcolm Smith \\
  University of Strathclyde \\
  \texttt{keith.smith@strath.ac.uk}
}

\begin{document}

\makeatletter
\@ifpackageloaded{backref}{%
  \renewcommand*{\backref}[1]{}%
  \renewcommand*{\backrefalt}[4]{}%
  \renewcommand*{\backrefsep}{}%
  \renewcommand*{\backreftwosep}{}%
  \renewcommand*{\backreflastsep}{}%
}{}
\makeatother

\setlength{\abovedisplayskip}{4pt plus 1pt minus 1pt}
\setlength{\belowdisplayskip}{4pt plus 1pt minus 1pt}
\setlength{\abovedisplayshortskip}{2pt plus 1pt}
\setlength{\belowdisplayshortskip}{3pt plus 1pt}
\setlength{\topsep}{3pt plus 1pt minus 1pt}
\setlength{\partopsep}{0pt}
\setlength{\parskip}{0pt}

\maketitle

\begin{abstract}
Many complex systems, including brain networks, financial markets, and gene-regulatory
circuits, are better described by interaction structures that evolve over time than by a
single fixed graph. The time-varying graphical lasso (TVGL) estimates this structure from
multivariate signals as a temporally coherent sequence of sparse precision matrices. We
introduce \textbf{TVGL--CFM}, a unified generative framework that learns distributions over
complete SPD precision-matrix trajectories without requiring a pre-specified graph, supporting
both class-conditional generation and history-conditioned forecasting. An SPD trajectory with
$T$ windows lies on the product Riemannian manifold $(\mathbb{S}_{++}^{p})^{T}$, and we
construct a global log-Euclidean diffeomorphism from this product space to a Euclidean sequence
space, letting a non-autoregressive conditional flow-matching model with a Transformer backbone
generate all windows jointly and decode to SPD matrices with no post-hoc projection. For
forecasting, we use two distinct data dependant couplings so the flow transforms an informative prior into a
coherent future block. Across EEG motor-imagery data and three nonlinear dynamical systems,
TVGL--CFM preserves class-discriminative dependency structure and forecasts future connectivity
more accurately than several strongly matched baselines opening up new possibilities for generative dynamic graph models. 
\end{abstract}

\section{Introduction}
A wide range of systems are most naturally described not by a static network but by a network
whose dependency structure changes over time: functional brain connectivity reconfigures
within a task or under stimulation, and cross-asset dependence in financial markets shifts
across regimes. When the signal at each node is (locally) approximately Gaussian, the canonical descriptor of such
structure is the precision matrix (inverse covariance), whose zeros encode conditional
independence. Estimating it is the graphical-lasso problem \citep{friedman2008,banerjee2008,yuan2007}, whose
time-varying extension, the time-varying graphical lasso
(TVGL) \citep{hallac2017}, estimates a chain of sparse precision matrices that is
both sparse at each time point and temporally coherent across time. When there is no known
underlying graph, which is common in various important applications, TVGL provides a robust
approach that produces plausible graph trajectories.

Diffusion models are a strong generative paradigm for complex data distributions
\citep{sohldickstein2015,ho2020,song2021}, and have been adapted to structured scientific
objects such as molecules, materials \citep{miller2024} and proteins \citep{huguet2024}; but
they have practical costs. Sampling typically needs many iterative denoising steps, and the
stochastic reverse process is expensive when the generated object is not a simple Euclidean
vector but a structured trajectory, graph, or matrix-valued object.

Conditional flow matching keeps the strengths of diffusion-style modelling while
simplifying training and sampling \citep{lipman2023,liu2023,albergo2023}. Instead of
denoising through a long stochastic reverse chain, flow matching learns a time-dependent
vector field that transports samples from a simple source distribution to the target along a
continuous-time ODE \citep{chen2018neuralode}. In the conditional
setting this transport can be guided by context, labels, or partial trajectories, which makes
it attractive for forecasting as the model can generate an entire future trajectory jointly,
avoiding autoregressive error accumulation and long sampling chains.

Trajectories of precision matrices lie on a Riemannian manifold rather than in Euclidean
space, which makes flow matching costly: it requires minimising along geodesics rather than
straight-line interpolants. Diffeomorphic flow matching bypasses this by mapping the manifold
through a global diffeomorphic chart, so that ordinary Euclidean flow-matching machinery can
be used. In particular,  \textsc{DiffeoCFM} \citep{collas2025} applied these diffeomorhisms to brain connectivity matrices which are usually represented as covariance or correlation matrices.  We extend this idea from a single matrix to a \emph{trajectory} of precision matrices, and show that the log-Euclidean chart is an exact diffeomorphism on the product
manifold. This yields an efficient paradigm, which we call TVGL--CFM: a generative model for
time-varying dynamical graphs, estimated directly from multivariate time series.

Two capabilities are valuable once such trajectories can be obtained. The first is
\textbf{generation}: sampling new, realistic, class-conditional trajectories, which is useful
when real recordings are scarce, as they often are in clinical neuroimaging
\citep{collas2025,marti2020}. The second is \textbf{forecasting}: given an observed prefix,
predicting how the trajectory continues, for example how a brain network will evolve in the
windows following a stimulation pulse \citep{fox2014,dennis2014}. We build on the geometric insight and extend it in
two directions that, to our knowledge, have not been combined before: from a single
matrix to a whole trajectory of matrices evolving over time, and from \emph{generation} to
\emph{forecasting} a trajectory's future from its past. We extend trajectory flow matching to
history-conditioned forecasting of an entire future precision trajectory, studying both a
mean-centred residual formulation and a random-walk stochastic source. Against baselines
that generate or
forecast raw EEG and estimate connectivity afterwards, modelling the precision trajectory
directly is markedly better on both tasks; against models trained on the identical TVGL
representation the margin is smaller but TVGL-CFM is still the best performing, outperforming even a strong dynamic graph transformer architecture modified to process TVGL representations. We further show transferability to a different temporal graphical model providing evidence that our framework can be used as a generalized tool.

\begin{contributionbox}
\begin{itemize}[leftmargin=1.2em,itemsep=0.35em,topsep=0.3em]

  \item \textbf{Flow matching over dynamic networks without a pre-specified graph.}
  We recover the latent dynamic network as a trajectory of time-varying sparse precision
  matrices with TVGL and learn to generate and forecast these trajectories directly, rather
  than modelling raw signals and estimating connectivity afterwards. For forecasting we study
  two data-dependent couplings: a mean-centred residual coupling, which models
  Gaussian residuals around a training-set future mean while conditioning the velocity field on
  the observed history, and a random-walk coupling whose stochastic source is itself built
  from recent increments of the observed trajectory.

  \item \textbf{A transformer-based, non-autoregressive flow-matching model on the SPD
  manifold.} A global log-Euclidean chart flattens an entire precision trajectory into a
  Euclidean coordinate, where one transformer velocity field, sampled non-autoregressively
  over the whole sequence, drives a single conditional flow-matching model and decodes to
  guaranteed-SPD matrices with no post-hoc projection.

  \item \textbf{Competitive generation and forecasting across domains.} The same geometric CFM
  framework supports class-conditional generation $p(\Th_{1:T}\mid y)$ and history-conditioned
  forecasting $p(\Th_{K+1:T}\mid \Th_{1:K},y)$, outperforming raw-signal baselines on every
  task and leading or matching models trained directly on the same precision trajectories.

\end{itemize}
\end{contributionbox}

\section{Background}
\subsection{SPD matrices and the log-Euclidean diffeomorphism}
Let $\Spd=\{\Sigma\in\R^{p\times p}\mid \Sigma^\top=\Sigma,\ \Sigma\succ 0\}$. The
matrix logarithm composed with half-vectorisation,
\begin{equation}
\phi(\Sigma) \;=\; \veclt\!\big(\logm \Sigma\big),
\qquad
\phi^{-1}(\eta) \;=\; \expm\!\big(\veclt^{-1}(\eta)\big),
\label{eq:diffeo}
\end{equation}
is a global diffeomorphism from $\Spd$ onto $\R^{p(p+1)/2}$, where $\veclt$ stacks the
lower-triangular entries with off-diagonal terms scaled by $\sqrt{2}$
\citep{arsigny2007,pennec2006}. Declaring $\phi$ distance-preserving defines the
\emph{log-Euclidean} metric on $\Spd$; equivalently, the Euclidean inner product on
$\R^{p(p+1)/2}$ is \emph{pulled back} through $\phi$ to a Riemannian metric on $\Spd$. It is
one of several structures proposed for SPD and correlation matrices
(Appendix~\ref{app:extended-bg}), and the one under which flow matching reduces to a Euclidean
problem. Because $\Th=\Sigma^{-1}$ has reciprocal eigenvalues and identical eigenvectors,
\begin{equation}
\phi(\Th) \;=\; -\,\phi(\Sigma),
\label{eq:negation}
\end{equation}
so flowing precision in this embedding is equivalent to flowing covariance up to a sign, and
discriminative information is preserved under the map.

\subsection{Conditional flow matching and diffeomorphic transport}
\label{sec:cfm-bg}
\label{sec:riemannian-diffeomorphic-fm}
Flow matching learns a time-dependent velocity field $u_\theta(s,z)$ that transports
samples from a simple source distribution $p_0=\mathcal N(0,\mathbf I)$ to a target
distribution $p_1$ along an ODE $\dot z=u_\theta(s,z)$
\citep{lipman2023,liu2023,albergo2023}, where $s\in[0,1]$ is the \emph{flow time}.
With the linear conditional path $z_s=(1-s)z_0+s\,z_1$ and target velocity $z_1-z_0$, the CFM
loss is
\begin{equation}
\mathcal L_{\mathrm{CFM}}(\theta)
=\mathbb{E}_{s,\,z_0\sim p_0,\,z_1\sim p_1}
\norm{u_\theta\big(s,(1-s)z_0+s\,z_1\big)-(z_1-z_0)}_2^2 .
\label{eq:cfm}
\end{equation}
Riemannian flow matching extends this to data on a manifold $\mathcal M$ \citep{chen2024},
but applying it directly to SPD precision trajectories requires manifold-valued paths and
Riemannian vector fields, at higher cost \citep{li2024spdddpm,desurrel2025}. Diffeomorphic
flow matching avoids this with a global diffeomorphism $\phi:\mathcal M\to E$ into a Euclidean
space, along whose pullback path the Riemannian CFM loss is exactly the Euclidean CFM loss in
the embedded coordinates $z=\phi(x)$ \citep{collas2025}. For TVGL--CFM, $\mathcal M$ is the
product SPD trajectory space $\Mtraj=(\mathcal S_{++}^{p})^T$ and $\phi$ is the log-Euclidean
embedding applied window-wise; Section~\ref{sec:flow} makes this precise.

\subsection{Time-varying graphical lasso}
\label{sec:tvgl-bg}
Given a multivariate signal windowed into $T$ slices with empirical covariances
$S_1,\dots,S_T$ (each from $n$ samples), TVGL \citep{hallac2017} estimates a chain of
sparse precision matrices $\Th=(\Th_1,\dots,\Th_T)$ by solving
\begin{equation}
\min_{\Th_1,\dots,\Th_T\in\Spd}\;
\sum_{i=1}^{T}\Big(-\log\det\Th_i+\tr(S_i\Th_i)\Big)
\;+\;\lambda\sum_{i=1}^{T}\norm{\Th_i}_{od,1}
\;+\;\beta\sum_{i=2}^{T}\psi\big(\Th_i-\Th_{i-1}\big),
\label{eq:tvgl}
\end{equation}
where $\norm{\cdot}_{od,1}$ is the off-diagonal $\ell_1$ norm and $\psi$ is a convex
evolutionary penalty controlling how the network may change between adjacent
timestamps. Following \citet{hallac2017} we consider three options for $\psi$, each encoding a
different assumption about how the network reconfigures between windows: writing
$D=\Th_i-\Th_{i-1}$, the element-wise $\ell_1$ penalty $\norm{D}_1$, the Laplacian penalty
$\norm{D}_F^2$, and the column group-lasso penalty $\sum_j\norm{D_{\cdot j}}_2$. TVGL
generalises static and jointly-estimated sparse graphical models
\citep{friedman2008,danaher2014,mohan2014} to an explicitly temporal chain; proximal operators
and precursors are in Appendices~\ref{app:extended-bg} and~\ref{app:tvgl-solver}.

\section{Method: TVGL--CFM}
\label{sec:method}
\begin{figure}[h]
    \centering
    \includegraphics[width=1.0\textwidth]{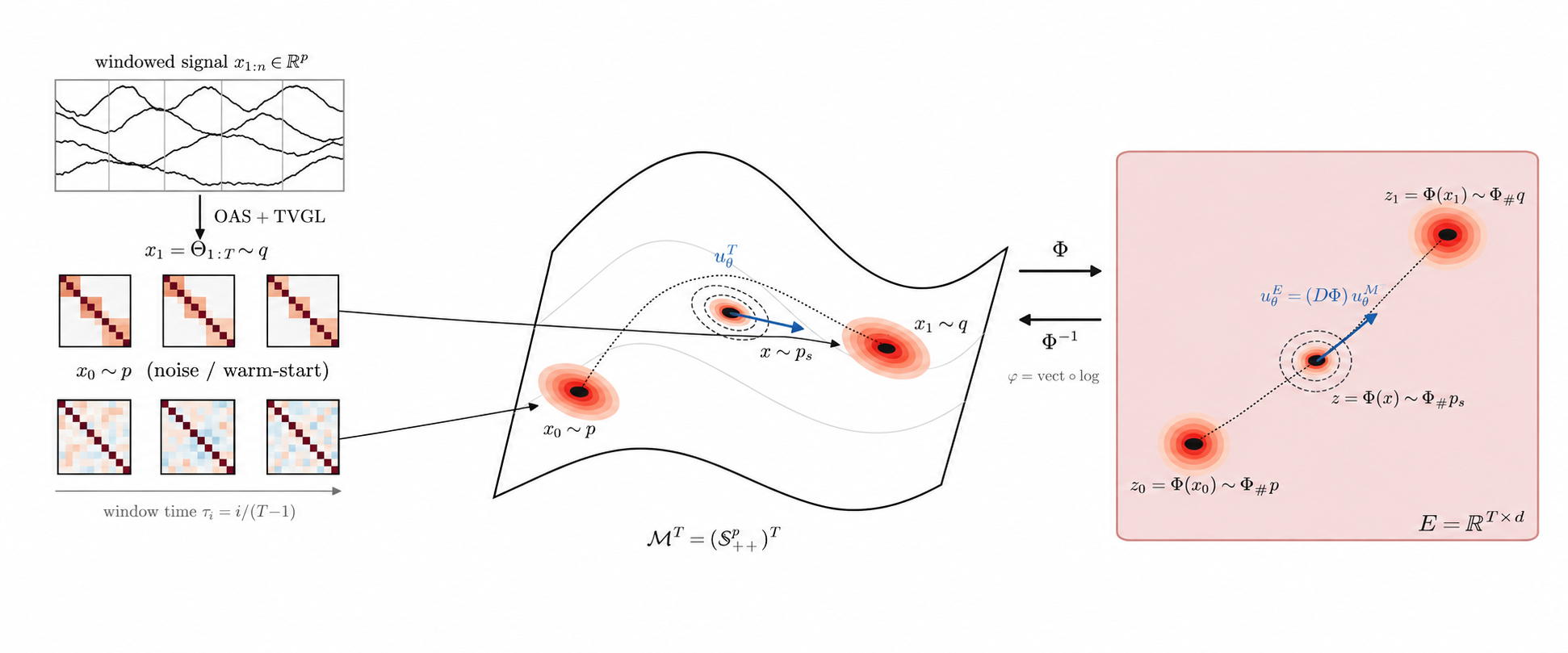}
    \caption{
    \textbf{Flow matching on TVGL precision-matrix trajectories via log-Euclidean embedding.}
    A multivariate windowed signal \(x_{1:n}\in\mathbb{R}^{p}\) is converted into a time-varying
    graphical trajectory \(\Theta_{1:T}\) using OAS covariance estimation followed by TVGL,
    defining samples from the target distribution \(q\) on the product SPD manifold
    \(\mathcal{M}^{T}=(\mathbb{S}^{p}_{++})^{T}\). A learned Riemannian vector field
    \(u_\theta^{\mathcal{M}^T}\) transports a source trajectory \(z_0\sim p\), implemented as
    Gaussian noise for generation, and under one of two data-dependent couplings for
    forecasting, toward a target \(z_1\sim q\). Trajectories are mapped
    through the log-Euclidean chart \(\Phi\) into \(E=\mathbb{R}^{T\times d}\), where the
    pushforwards are connected by \(u_\theta^{E}=(D\Phi)u_\theta^{\mathcal{M}^{T}}\), enabling
    conditional flow matching while preserving the SPD structure after decoding with
    \(\Phi^{-1}\).
    }
    \label{fig:tvglcfm_manifold_embedding}
\end{figure}
\paragraph{Problem setup.}
A dataset is a collection of $p$-channel sequences with optional class labels $y$. Each
sequence is windowed into $T$ slices and reduced, per window, to a TVGL precision matrix,
giving a per-sequence trajectory $\Th_{1:T}\in(\Spd)^T$. We address two tasks on these
trajectories: \emph{generation} of a full trajectory given a class, $p(\Th_{1:T}\mid y)$, and
\emph{forecasting} of the future given an observed prefix of length $K$,
$p(\Th_{K+1:T}\mid\Th_{1:K},y)$. Both share the target-construction and embedding stages and
differ only in the conditioning of the flow
(Figures~\ref{fig:tvglcfm_manifold_embedding} and~\ref{fig:tvglcfm_generation_forecasting},
Appendix~\ref{app:workflow}).

\subsection{TVGL targets, embedding and standardisation}
\label{sec:tvgl}
\label{sec:embedding}
Each sequence's $p$-channel signal is $z$-scored per channel over time and split into $T$
contiguous windows of length $n$. Because $n$ may be comparable to $p$, the raw sample covariance is poorly conditioned; we
use the oracle-approximating-shrinkage (OAS) estimator \citep{chen2010}, and divide every
estimate by the mean training diagonal so that the TVGL $\lambda$ is interpretable. We then solve \eqref{eq:tvgl} for all
sequences and windows simultaneously with ADMM \citep{boyd2011,boyd2004}, whose updates have
closed forms in terms of proximal operators \citep{parikh2014}. Appendix~\ref{app:tvgl-solver} gives the derivation and Algorithm~\ref{alg:tvgl}; the output
is a sparse SPD precision trajectory per sequence.

We embed each window of the target with the log-Euclidean diffeomorphism \eqref{eq:diffeo},
applied independently across the trajectory:
\begin{equation}
\Phi(\Theta_{1:T})=\big(\phi(\Theta_1),\ldots,\phi(\Theta_T)\big)\in E=\mathbb R^{T\times d},
\qquad d=p(p+1)/2 .
\label{eq:trajectory-embedding}
\end{equation}
These coordinates are not equally scaled: diagonal log-precision entries are typically much
larger than off-diagonal edge entries, so training CFM directly on $z_n$ would let
high-variance coordinates dominate the squared loss. We therefore standardise each coordinate
with training-set statistics $(\mu_k,\sigma_k)$, an affine diffeomorphism $\Psi:E\to E$
composed with the chart to give $\widetilde{\Phi}=\Psi\circ\Phi$:
\begin{equation}
\Psi(z)_{i,k}=\frac{z_{i,k}-\mu_k}{\sigma_k},
\qquad \Psi^{-1}(\widetilde z)_{i,k}=\sigma_k\widetilde z_{i,k}+\mu_k,
\qquad \widetilde z=\widetilde{\Phi}(\Theta_{1:T}) .
\label{eq:embedding-standardisation}
\end{equation}
The geometric validity of this embedding follows from the product structure of the
trajectory manifold.

\begin{proposition}[Product diffeomorphism and pullback metric]
\label{prop:prod-diffeo}
$\Phi:\Mtraj\to E$ is a global diffeomorphism, and the product log-Euclidean
metric \eqref{eq:product-le-metric} is the pullback $\Phi^{*}g_{E}$ of the
standard Euclidean metric on $E$.
\end{proposition}

Standardisation simply composes this product diffeomorphism with an invertible affine
map.

\begin{proposition}[Standardised diffeomorphism]
\label{prop:standardisation}
Assume $\sigma_k>0$ for all $k=1,\ldots,d$. Then
$\widetilde{\Phi}:=\Psi\circ\Phi:\Mtraj\to E$ is a global diffeomorphism. Its
pullback metric is the per-coordinate rescaling
\begin{equation}
\label{eq:standardised-metric}
  (\widetilde{\Phi}^{*}g_{E})_{\Theta_{1:T}}(\xi,\eta)
  =
  \sum_{i=1}^{T}\sum_{k=1}^{d}
   \sigma_{k}^{-2}\,
   \bigl[D\phi(\Theta_{i})[\xi_{i}]\bigr]_{k}\,
   \bigl[D\phi(\Theta_{i})[\eta_{i}]\bigr]_{k},
\end{equation}
for tangent trajectories
$\xi,\eta\in T_{\Theta_{1:T}}\Mtraj$.
\end{proposition}

Proofs are in Appendix~\ref{app:pullback}. Euclidean CFM in the standardised coordinates
$\widetilde z=\widetilde{\Phi}(\Theta_{1:T})$ is therefore still a pullback Riemannian CFM
model on the original SPD trajectory manifold, and the equivalences of
Section~\ref{sec:flow} hold with $\widetilde{\Phi}$ in place of $\Phi$. After sampling we
invert the standardisation and decode each window through the inverse log-Euclidean map:
\begin{equation}
\widehat{\Theta}_{i}=\phi^{-1}\!\left(\sigma\odot \widehat{\widetilde z}_{i}+\mu\right),
\qquad i=1,\ldots,T .
\label{eq:decode-standardised}
\end{equation}
Because $\phi^{-1}$ is the matrix exponential after inverse half-vectorisation, every
decoded $\widehat{\Theta}_i$ is SPD by construction. It is not, however, exactly sparse; but there are generated values close to 0 instead thus an optional sparsified readout is discussed in  the Appendix (~\ref{rem:spd-not-sparse}).All metrics are computed on the unsparsified output.

\subsection{Generation: joint-trajectory conditional flow}
\label{sec:flow}
A single transformer encoder $u_\theta$ \citep{vaswani2017} processes the whole standardised
trajectory as a length-$T$ token sequence, each token summed with three embeddings: flow time
$s$, window time $\tau_i=i/(T{-}1)$ and class $y$, the continuous times entering through random
Fourier features \citep{tancik2020}. Writing $z\in\R^{T\times d}$,
\begin{equation}
u_\theta(z,s,y,\tau)=\mathrm{Dec}\Big(\mathrm{TX}\big(\mathrm{Enc}(z)
+\mathrm{FF}_s(s)+\mathrm{FF}_\tau(\tau)+\mathrm{Emb}_y(y)\big)\Big).
\label{eq:ufield}
\end{equation}
Here $\mathrm{Enc}:\R^{d}\to\R^{d_{\mathrm{model}}}$ is a linear token embedding applied per
window, $\mathrm{FF}_s$ and $\mathrm{FF}_\tau$ are random-Fourier-feature maps of the two
continuous times followed by an MLP, $\mathrm{Emb}_y$ is a learned class embedding,
$\mathrm{TX}$ is a Transformer encoder stack over the $T$ window tokens, and
$\mathrm{Dec}:\R^{d_{\mathrm{model}}}\to\R^{d}$ is a linear read-out back to chart
coordinates. Both CFM endpoints are whole trajectories: $E=\R^{T\times d}\cong\R^{Td}$, so
$z_0\sim\mathcal N(0,I_{Td})$ and $z_1=\widetilde\Phi(\Th_{1:T})$ are single points in one
trajectory-level space and $z_s=(1-s)z_0+s\,z_1$ interpolates the whole trajectory at once.
Because the standardised chart $\widetilde{\Phi}$ is a global diffeomorphism onto
$E$ with pullback metric $\widetilde{\Phi}^{*}g_{E}$
(Propositions~\ref{prop:prod-diffeo}--\ref{prop:standardisation}), the three
equivalences below hold with $\widetilde{\Phi}$ in place of $\Phi$; we state them
for $\Phi$ to lighten notation. The first shows that training the field in the
embedding is exactly Riemannian CFM on the trajectory manifold. Throughout, the
Euclidean field is the pushforward of a Riemannian field
$u^{\Mtraj}_{\theta}$ on $(\Mtraj,\Phi^{*}g_{E})$ through the differential
$D\Phi$ (Eq.~\eqref{eq:pullback-field-traj}).

\begin{proposition}[Trajectory CFM loss equivalence]
\label{prop:traj-cfm}
Let $p,q$ be class-conditional measures on $\Mtraj$ and write
$z_{0}\mid y\sim\Phi_{\#}p(\cdot\mid y)$, $z_{1}\mid y\sim\Phi_{\#}q(\cdot\mid y)$
in $E$. Define the Euclidean field
\begin{equation}
\label{eq:pullback-field-traj}
  u^{E}_{\theta}(s,z,y)
  \;:=\;D\Phi\bigl(\Phi^{-1}(z)\bigr)\bigl[u^{\Mtraj}_{\theta}(s,\Phi^{-1}(z),y)\bigr].
\end{equation}
Then the Riemannian CFM loss on $(\Mtraj,\Phi^{*}g_{E})$ equals
\begin{equation}
\label{eq:traj-cfm-loss}
  \Ltraj(\theta)
  \;=\;\E_{s,\,y,\,z_{0}\mid y,\,z_{1}\mid y}\bigl\|
    u^{E}_{\theta}\bigl(s,(1-s)z_{0}+s\,z_{1},y\bigr)-(z_{1}-z_{0})
  \bigr\|_{E}^{2}.
\end{equation}
The proof is given in Appendix~\ref{app:proof-traj-cfm}.
\end{proposition}

The next result establishes that integrating the learned Euclidean ODE and
decoding through $\Phi^{-1}$ reproduces the manifold-valued flow exactly.

\begin{proposition}[Trajectory ODE equivalence]
\label{prop:traj-ode}
Let $z:[0,1]\to E$ solve $\dot z(s)=u^{E}_{\theta}(s,z(s),y)$ with
$z(0)=z_{0}=\Phi(\Theta_{1:T}^{0})$. Then
$\Theta_{1:T}(s):=\Phi^{-1}(z(s))$ solves the Riemannian ODE
\begin{equation}
\label{eq:riem-ode-traj}
  \dot\Theta_{1:T}(s)=u^{\Mtraj}_{\theta}(s,\Theta_{1:T}(s),y),
  \qquad
  \Theta_{1:T}(0)=\Theta_{1:T}^{0},
\end{equation}
on $(\Mtraj,\Phi^{*}g_{E})$, and $\Theta_{i}(s)\in\Spd$ for every
$i\in\{1,\dots,T\}$ and every $s\in[0,1]$. The proof is given in
Appendix~\ref{app:proof-traj-ode}.
\end{proposition}

The equivalence also holds at the discrete level: the same explicit Runge--Kutta
scheme applied in $E$ or on $\Mtraj$ produces iterates related by $\Phi$, so the
practical RK4 sampler of Algorithm~\ref{alg:sample} inherits the guarantees above.

\begin{proposition}[Trajectory Runge--Kutta equivalence]
\label{prop:traj-rk}
Let an explicit Runge--Kutta scheme with $\nu$ stages, Butcher tableau
$(a_{ij},b_{i},c_{i})$, and step size $h$ be applied to the Riemannian ODE
\eqref{eq:riem-ode-traj} on $(\Mtraj,\Phi^{*}g_{E})$ and to the Euclidean ODE
$\dot z=u^{E}_{\theta}(s,z,y)$ on $E$, with initial conditions
$\Theta_{1:T}^{0}$ and $z^{0}=\Phi(\Theta_{1:T}^{0})$ respectively. Then
\begin{equation}
\label{eq:rk-equivalence-traj}
  \Theta_{1:T}^{\ell}\;=\;\Phi^{-1}(z^{\ell})
  \qquad\text{for all }\ell\in\N,
\end{equation}
and every iterate $\Theta_{1:T}^{\ell}$ is slice-wise SPD. The proof is given in
Appendix~\ref{app:proof-traj-rk}.
\end{proposition}

\paragraph{TVGL temporal regulariser.}
For a path point $z_s$ and predicted velocity $u_\theta$, the implied clean endpoint is
$\hat z_1=z_s+(1-s)\,u_\theta$. Writing $\hat z_{1,i}$ for the $i$-th window slice of this
clean endpoint, we bias the generator toward the chosen evolutionary pattern by penalising
the corresponding $\psi$ on consecutive windows of $\hat z_1$:
\begin{equation}
\mathcal L(\theta)=\mathcal L_{\mathrm{CFM}}(\theta)
+\lambda_{\mathrm{temp}}\;\mathbb{E}\!\left[\frac{1}{T-1}\sum_{i=2}^{T}
\psi_E\big(\hat z_{1,i}-\hat z_{1,i-1}\big)\right],
\label{eq:loss}
\end{equation}
with $\psi_E(\cdot)=\norm{\cdot}_1$ for the $\ell_1$ penalty and
$\psi_E(\cdot)=\norm{\cdot}_2^2$ for the smooth/group penalties. The sum is over
\emph{window time} at fixed flow time $s=1$; the CFM term $\mathcal L_{\mathrm{CFM}}$
itself averages over flow time. The main temporal structure is carried by the targets; \eqref{eq:loss} is a soft prior
consistent with TVGL. Training and sampling are given in
Algorithms~\ref{alg:train}--\ref{alg:sample}.

\begin{remark}
\label{rem:coupled-dynamics}
Since $\Phi=\phi\times\cdots\times\phi$ its differential is block diagonal, so the \emph{chart}
acts factor-wise. That concerns coordinates only: the components $u_{\theta,i}(\Th_{1:T})$ of
the learned field each depend on the \emph{entire} trajectory through attention, so the
geometry is a Cartesian product while the dynamics are globally coupled. A window-wise
extension of \textsc{DiffeoCFM} \citep{collas2025} would transport each matrix independently,
modelling roughly $\prod_{t}p(\Th_t\mid y)$; we model $p(\Th_{1:T}\mid y)$ directly.
\end{remark}


\subsection{Forecasting: history-conditioned future flow}
\label{sec:forecast}
We split each standardised trajectory into an observed history $z_{1:K}$ and future
$z^{\mathrm f}:=z_{K+1:T}$, and learn a conditional flow over the future block. The model is
non-autoregressive: it observes the prefix and samples the whole future jointly, so future
windows attend to one another without autoregressive error accumulation.

\paragraph{History conditioning.}
A $2$-layer Transformer encodes the prefix into one token per window,
\begin{equation}
H=\mathcal E_\theta(z_{1:K})
=\mathrm{TX}_{\mathrm{hist}}\big(\mathrm{Enc}(z_{1:K})
+\mathrm{FF}_\tau(\tau_{1:K})+\mathrm{Emb}_y(y)\big)\in\R^{K\times d_{\mathrm{model}}},
\label{eq:histenc}
\end{equation}
with $\mathrm{Enc}$, $\mathrm{FF}_\tau$ and $\mathrm{Emb}_y$ as in \eqref{eq:ufield}, omitting
the last when no label is available. The future stack cross-attends to $H$, so each future
window can attend to every observed window.

\paragraph{Source, path and field.}
As in generation, the endpoints are whole blocks in
$\Ef=\R^{(T-K)\times d}\cong\R^{(T-K)d}$. Writing $z^{\mathrm f}_1$ for the target at
$s=1$ and $z^{\mathrm f}_0$ for the source at $s=0$,
\begin{equation}
z^{\mathrm f}_s=(1-s)z^{\mathrm f}_0+s\,z^{\mathrm f}_1,
\qquad
v^{\star}_s=z^{\mathrm f}_1-z^{\mathrm f}_0 ,
\label{eq:forecastpath}
\end{equation}
and the field is
\begin{equation}
\begin{gathered}
u_\theta\big(z^{\mathrm f}_s,s,y,\tau_{K+1:T},H\big)
=\mathrm{Dec}\Big(\mathrm{TX}_{\mathrm f}\big(Z_s\,\big|\,H\big)\Big),\\[2pt]
Z_s=\mathrm{Enc}(z^{\mathrm f}_s)+\mathrm{FF}_s(s)
+\mathrm{FF}_\tau(\tau_{K+1:T})+\mathrm{Emb}_y(y).
\end{gathered}
\label{eq:futurefield}
\end{equation}
Here $\mathrm{TX}_{\mathrm f}(\cdot\mid H)$ self-attends over the $T-K$ future tokens and
cross-attends to all history tokens in $H$. Section~\ref{sec:informedsources} specifies the
coupling, i.e.\ $z^{\mathrm f}_0$ and $z^{\mathrm f}_1$.

The field is trained with two additional losses, both evaluated on the absolute future
$\hat z^{\mathrm f}_1$ implied by the predicted velocity. The temporal regulariser
\eqref{eq:loss} is applied to $[\,z_{1:K};\hat z^{\mathrm f}_1\,]$, enforcing coherence across
the history--future handoff. We also match the first future transition:
\begin{equation}
\mathcal L_{\mathrm{bnd}}
=\mathbb{E}\,
\norm{\big(\hat z^{\mathrm f}_{1,1}-z_K\big)-\big(z^{\mathrm f}_{1,1}-z_K\big)}_2^2
=\mathbb{E}\,\norm{\hat z^{\mathrm f}_{1,1}-z^{\mathrm f}_{1,1}}_2^2 ,
\label{eq:bnd}
\end{equation}
a squared log-Euclidean distance between the decoded matrices
(Proposition~\ref{prop:boundary}). The total objective is
\begin{equation}
\mathcal L_{\mathrm{forecast}}
=\mathbb{E}_{(z^{\mathrm f}_0,z^{\mathrm f}_1)\sim\pi,\,y,\,s}
\norm{u_\theta(z^{\mathrm f}_s,s,y,\tau_{K+1:T},H)
-(z^{\mathrm f}_1-z^{\mathrm f}_0)}_2^2
+\lambda_{\mathrm{temp}}\mathcal R_{\mathrm{temp}}
+\lambda_{\mathrm{bnd}}\mathcal L_{\mathrm{bnd}},
\label{eq:residforecastloss}
\end{equation}
where $\pi$ is the source--target coupling. Decoding inverts the standardisation and applies
$\phi^{-1}$ window-wise, so every forecast is slice-wise SPD without projection.

\subsection{Data-dependent couplings}
\label{sec:informedsources}
The standard choice $z^{\mathrm f}_0\sim\mathcal N(0,I)$,
$z^{\mathrm f}_1=z^{\mathrm f}$ ignores information already contained in the history and
training set. Guided flow matching instead uses an informative source so the flow models the
remaining correction \citep{xu2025cgfm}. We therefore consider two source--target couplings
$\pi$ acting on opposite ends of the pair; Proposition~\ref{prop:forecast-cfm} shows the CFM
equivalence continues to hold for general couplings.

\paragraph{(a) Mean-centred residual coupling.}
From the training split only, define the mean future block in standardised coordinates,
\begin{equation}
\mf=\frac{1}{N}\sum_{n=1}^{N} z_{n,K+1:T}\in\R^{(T-K)\times d},
\label{eq:classmean}
\end{equation}
and model deviations around it:
\begin{equation}
z^{\mathrm f}_1=z^{\mathrm f}-\mf,\qquad
z^{\mathrm f}_0\sim\mathcal N(0,I).
\label{eq:residual}
\end{equation}
Thus only the target is translated; the Gaussian source is unchanged, and
\eqref{eq:forecastpath}--\eqref{eq:futurefield} operate in residual coordinates. After
integration,
$\hat z^{\mathrm f}=\mf+\hat z^{\mathrm f}_1$, so a zero residual returns $\mf$.

\paragraph{(b) Random-walk coupling.}
Here the target remains absolute, $z^{\mathrm f}_1=z^{\mathrm f}$, while the source is a
stochastic extrapolation of the observed prefix, following data-coupled trajectory simulation
\citep{bennema2026stflow}. From the last $r$ increments, estimate elementwise drift and spread,
\begin{equation}
\widehat\mu=\frac1r\sum_{j=K-r+1}^{K}(z_j-z_{j-1}),
\qquad
\widehat\sigma=
\sqrt{\frac1r\sum_{j=K-r+1}^{K}
\big(z_j-z_{j-1}-\widehat\mu\big)^{\odot 2}}+\sigma_{\min},
\label{eq:rwstats}
\end{equation}
then, from $\tilde z^{\mathrm f}_{0,0}=z_K$, sample
\begin{equation}
\tilde z^{\mathrm f}_{0,q}
=\tilde z^{\mathrm f}_{0,q-1}
+\alpha_\mu\widehat\mu
+\alpha_\sigma\widehat\sigma\odot\epsilon_q,
\qquad
\epsilon_q\sim\mathcal N(0,I),\quad q=1,\dots,T-K.
\label{eq:rwprior}
\end{equation}
We set
\[
z^{\mathrm f}_0=
(\tilde z^{\mathrm f}_{0,1},\dots,\tilde z^{\mathrm f}_{0,T-K})
\sim p_{\mathrm{RW}}(\cdot\mid z_{1:K}).
\]
This random walk evolves through graph-state space: each $z_j$ represents an entire precision
matrix via $z_j=\phi(\Th_j)$, so decoding any state remains SPD. Setting the informed source to
plain Gaussian recovers the standard CFM construction.
\section{Experiments}
\label{sec:exp}
We evaluate on four MOABB motor-imagery datasets \citep{jayaram2018moabb} (BNCI2014\_001,
BNCI2014\_002, BNCI2015\_001, Zhou2016) and on three simulated systems that probe domain
generality: \textbf{Lorenz}, diffusively coupled Lorenz oscillators; \textbf{MacArthur}, a
consumer-resource competition system; and \textbf{Hopfield}, a continuous associative-memory
network. Preprocessing, architecture, baselines and metrics are in
Appendices~\ref{app:experimental-details}--\ref{app:implementation}. Every result is a mean
over three seeds. Each sample is the whole chain (for forecasting, the future block), embedded
per window and concatenated into a single vector, so metrics are sensitive to temporal
arrangement.

\subsection{Generating Realistic Brain-Network Dynamics}

\begin{table*}[t]
\centering
\caption{Three-seed generation performance on EEG datasets. Entries are means over three random seeds; standard deviations are in Appendix~\ref{app:eeg-std}. CAS AUC and F1 measure train-on-generated $\rightarrow$ test-on-real classifiability (higher is better), while Rel-FD is the Fr\'echet distance between generated and held-out real precision trajectories normalized by the real train--test reference (lower is better). Direct precision-space models generate TVGL trajectories directly; raw-EEG baselines generate signals that are subsequently mapped through the same OAS+TVGL pipeline.}
\label{tab:eeg_generation}

\renewcommand{\arraystretch}{1.35}
\setlength{\tabcolsep}{3.0pt}

\resizebox{\textwidth}{!}{
\begin{tabular}{lccccccccccccccc}
\toprule
& \multicolumn{3}{c}{Zhou2016}
& \multicolumn{3}{c}{BNCI2014\_001}
& \multicolumn{3}{c}{BNCI2014\_002}
& \multicolumn{3}{c}{BNCI2015\_001}
& \multicolumn{3}{c}{Mean} \\
\cmidrule(lr){2-4}
\cmidrule(lr){5-7}
\cmidrule(lr){8-10}
\cmidrule(lr){11-13}
\cmidrule(lr){14-16}
Method
& AUC & F1 & Rel-FD
& AUC & F1 & Rel-FD
& AUC & F1 & Rel-FD
& AUC & F1 & Rel-FD
& AUC & F1 & Rel-FD \\
\midrule

TVGL oracle
& 0.965 & 0.900 & --
& 0.794 & 0.722 & --
& 0.777 & 0.717 & --
& 0.762 & 0.692 & --
& -- & -- & -- \\

\midrule

\textbf{TVGL--CFM}
& \textbf{0.954}
& \textbf{0.893}
& \underline{\textit{1.101}}
& \textbf{0.718}
& \textbf{0.646}
& \textbf{1.191}
& \textbf{0.736}
& \textbf{0.678}
& \underline{\textit{1.022}}
& \textbf{0.717}
& \textbf{0.649}
& \textbf{1.061}
& \textbf{0.781}
& \textbf{0.716}
& \textbf{1.094} \\

TVGL--cVAE
& 0.879
& 0.799
& 1.111
& \underline{\textit{0.677}}
& \underline{\textit{0.620}}
& 1.938
& \underline{\textit{0.680}}
& \underline{\textit{0.621}}
& \textbf{0.994}
& \underline{\textit{0.693}}
& \underline{\textit{0.635}}
& 2.725
& \underline{\textit{0.732}}
& \underline{\textit{0.669}}
& 1.692 \\

TVGL--DDPM
& 0.889
& 0.792
& \textbf{1.075}
& 0.615
& 0.571
& \underline{\textit{1.572}}
& 0.624
& 0.587
& 1.098
& 0.603
& 0.569
& \underline{\textit{1.223}}
& 0.683
& 0.629
& \underline{\textit{1.242}} \\

Window-DIFFEO-CFM
& \underline{\textit{0.933}}
& \underline{\textit{0.853}}
& 1.208
& 0.626
& 0.596
& 2.145
& 0.661
& 0.621
& 1.386
& 0.664
& 0.593
& 1.536
& 0.721
& 0.666
& 1.569 \\

\midrule

JET $\rightarrow$ TVGL
& 0.573
& 0.398
& 9.216
& 0.508
& 0.385
& 11.374
& 0.506
& 0.356
& 7.845
& 0.505
& 0.334
& 31.738
& 0.523
& 0.368
& 15.043 \\

cVAE $\rightarrow$ TVGL
& 0.880
& 0.427
& 9.700
& 0.575
& 0.417
& 9.983
& 0.522
& 0.430
& 6.156
& 0.589
& 0.456
& 28.113
& 0.642
& 0.433
& 13.488 \\

cWGAN-GP $\rightarrow$ TVGL
& 0.678
& 0.570
& 2.602
& 0.520
& 0.513
& 3.498
& 0.507
& 0.500
& 2.018
& 0.510
& 0.477
& 8.523
& 0.554
& 0.515
& 4.160 \\

EEG-GAN $\rightarrow$ TVGL
& 0.375
& 0.335
& 8.288
& 0.537
& 0.373
& 12.802
& 0.547
& 0.391
& 4.474
& 0.566
& 0.417
& 5.792
& 0.507
& 0.379
& 7.839 \\

\bottomrule
\end{tabular}}
\end{table*}

We assess downstream utility using the classification-accuracy score (CAS), training a small MLP
on generated trajectories and evaluating on held-out real trials \citep{ravuri2019}. Generation
quality is measured by relative Fr\'echet distance (Rel-FD) in the log-Euclidean embedding.

TVGL--CFM is strongest overall, achieving the highest CAS AUC/F1 on all four datasets and the
best mean AUC, F1 and Rel-FD ($0.781$, $0.716$, $1.094$). TVGL--DDPM attains a slightly lower
Rel-FD on Zhou2016 ($1.075$ vs.\ $1.101$), and TVGL--cVAE on BNCI2014\_002
($0.994$ vs.\ $1.022$).

The gap is much larger against raw-EEG generators decoded through TVGL: TVGL--CFM gains
$0.074$--$0.189$ AUC per dataset, with Rel-FD $1.022$--$1.191$ versus $2.018$--$31.738$.
JET \citep{jet} reaches only $0.523$ mean AUC at Rel-FD $15.043$, indicating that realistic
signal generation alone does not preserve temporal dynamics. Against models trained on the same
TVGL trajectory representation, the margin is smaller but consistent: TVGL--CFM leads
TVGL--cVAE, TVGL--DDPM and Window-DIFFEO-CFM in AUC on every dataset and in mean Rel-FD
($1.094$ vs.\ $1.692$, $1.242$, $1.569$). Window-DIFFEO-CFM shares the targets and chart but
removes temporal communication, dropping to $0.721$ mean AUC at Rel-FD $1.569$, supporting
joint trajectory modelling. The mean AUC gap to the TVGL oracle is only $0.044$, while every
generated matrix remains SPD by construction.

\subsection{Forecasting Future Brain-Network Dynamics}
\label{sec:eeg-forecast}


\begin{table*}[t]
\centering
\caption{Three-seed forecasting performance on EEG datasets. Entries are means over three random seeds; standard deviations are in Appendix~\ref{app:eeg-std}. Lower is better for AIRM and log-Euclidean RMSE (logE). Both TVGL--CFM couplings are reported: the mean-centred residual coupling and the random-walk coupling. The remaining learned forecasters predict the absolute future directly. Best in \textbf{bold}, second best \underline{\textit{underlined and italicized}}.}
\label{tab:eeg_forecasting}

\renewcommand{\arraystretch}{1.35}
\setlength{\tabcolsep}{3.0pt}

\resizebox{0.94\textwidth}{!}{
\begin{tabular}{lcccccccccc}
\toprule
& \multicolumn{2}{c}{Zhou2016}
& \multicolumn{2}{c}{BNCI2014\_001}
& \multicolumn{2}{c}{BNCI2014\_002}
& \multicolumn{2}{c}{BNCI2015\_001}
& \multicolumn{2}{c}{Mean} \\
\cmidrule(lr){2-3}
\cmidrule(lr){4-5}
\cmidrule(lr){6-7}
\cmidrule(lr){8-9}
\cmidrule(lr){10-11}
Method
& AIRM & logE
& AIRM & logE
& AIRM & logE
& AIRM & logE
& AIRM & logE \\
\midrule

\textbf{TVGL--CFM (mean-residual)}
& \textbf{2.415}
& \textbf{2.300}
& 2.747
& 2.544
& \textbf{1.552}
& \textbf{1.417}
& \underline{\textit{1.952}}
& \underline{\textit{1.793}}
& \textbf{2.167}
& \textbf{2.014} \\

TVGL--CFM (Random-Walk)
& 2.937
& 2.773
& 3.445
& 3.149
& 1.817
& 1.636
& 2.167
& 1.979
& 2.591
& 2.384 \\

TVGL--cVAE-Forecast
& \underline{\textit{2.452}}
& \underline{\textit{2.337}}
& \underline{\textit{2.706}}
& \textbf{2.503}
& 1.601
& 1.456
& 1.977
& 1.809
& 2.184
& 2.026 \\

TVGL--DDPM-Forecast
& 2.768
& 2.687
& 3.361
& 3.206
& 1.797
& 1.660
& 2.140
& 2.000
& 2.517
& 2.388 \\

DyGraphformer-T
& 2.460
& 2.348
& \textbf{2.696}
& \underline{\textit{2.507}}
& \underline{\textit{1.578}}
& \underline{\textit{1.441}}
& \textbf{1.940}
& \textbf{1.780}
& \underline{\textit{2.168}}
& \underline{\textit{2.019}} \\

Training mean (no flow)
& 2.847
& 2.769
& 3.286
& 3.143
& 1.839
& 1.707
& 2.203
& 2.067
& 2.544
& 2.421 \\

Persistence
& 3.011
& 2.821
& 3.362
& 3.046
& 1.931
& 1.720
& 2.520
& 2.272
& 2.706
& 2.465 \\

\midrule

Basic-CNN $\rightarrow$ TVGL
& 9.843
& 9.884
& 10.897
& 10.943
& 5.920
& 5.946
& 8.739
& 8.764
& 8.850
& 8.884 \\

cVAE $\rightarrow$ TVGL
& 8.392
& 8.422
& 9.098
& 9.120
& 3.143
& 3.097
& 6.046
& 6.029
& 6.670
& 6.667 \\

cWGAN-GP $\rightarrow$ TVGL
& 9.809
& 9.851
& 10.735
& 10.780
& 4.902
& 4.925
& 5.342
& 5.279
& 7.697
& 7.709 \\

EEG-GAN $\rightarrow$ TVGL
& 4.327
& 4.101
& 5.117
& 4.742
& 2.978
& 2.773
& 3.193
& 2.972
& 3.904
& 3.647 \\

JET $\rightarrow$ TVGL
& 3.861
& 3.490
& 4.126
& 3.891
& 3.946
& 3.059
& 4.409
& 3.890
& 4.085
& 3.582 \\

\bottomrule
\end{tabular}}
\end{table*}
We next evaluate whether the proposed Forecast TVGL--CFM can predict future brain-network
dynamics from an observed prefix of a held-out EEG trial. For each trial TVGL estimates a
sequence of precision matrices $\Theta_{1:T}$, and the model observes the first $K=15$ windows
and forecasts the remaining $T-K=5$ future windows,
$p_\theta(\Theta_{K+1:T} \mid \Theta_{1:K}, y)$.

The training mean alone reaches $2.544$ mean AIRM, ahead of persistence at $2.707$. Adding the
residual flow lowers the error on every dataset and both metrics: $3.287$ to $2.748$ on
BNCI2014\_001, $2.847$ to $2.415$ on Zhou2016, $2.204$ to $1.953$ on BNCI2015\_001,
$1.840$ to $1.553$ on BNCI2014\_002, and $2.544$ to $2.167$ in mean AIRM, a $15\%$ reduction. The
gain therefore comes from the learned history-conditioned transport rather than the mean
future. All four raw-EEG forecasters decoded through TVGL are beaten on both metrics and all
four datasets; the best of them, EEG-GAN, is $1.80\times$ worse in mean AIRM ($3.904$ against
$2.167$).

Among models trained on the identical TVGL representation the differences are smaller. TVGL--CFM has
the lowest mean AIRM ($2.167$) and mean logE-RMSE ($2.014$) and the lowest values on Zhou2016
and BNCI2014\_002 on both metrics. DyGraphformer-T, a graph-native dynamic-network forecaster
on the same targets, has the lowest AIRM on BNCI2014\_001 and BNCI2015\_001 and is $0.002$
behind on mean AIRM; TVGL--cVAE-Forecast is an additional $0.018$ back.
TVGL--CFM additionally generates class-conditional trajectories under the same objective.

The two couplings differ substantially on EEG. The mean-centred residual coupling reaches
$2.167$ mean AIRM and the random-walk coupling $2.592$, which is worse than the mean
alone. Here the mean is taken per class.

\begin{table}[!t]
\centering
\caption{\textbf{Forecasting results across chaotic systems}, averaged over all nine
system$\times$horizon settings (Lorenz, MacArthur and Hopfield at $h\in\{3,5,7\}$, fixed
history $K=15$, three seeds each). AIRM is the primary metric and logE-RMSE the secondary;
both are lower-is-better. Both TVGL--CFM source constructions are reported. Per-system and
per-horizon results are in Appendix~\ref{app:full-results}. Best in \textbf{bold}, second best
\underline{\textit{underlined and italicized}}.}
\label{tab:forecast_main}
\footnotesize
\renewcommand{\arraystretch}{1.15}
\setlength{\tabcolsep}{9pt}
\begin{tabular}{lcc}
\toprule
Method & Mean AIRM & Mean logE-RMSE \\
\midrule
\textbf{TVGL--CFM (Random-Walk)} & \textbf{1.069} & \textbf{1.162} \\
TVGL--CFM (mean-residual)  & 1.166 & 1.284 \\
TVGL--cVAE-Forecast                     & \underline{\textit{1.129}} & \underline{\textit{1.188}} \\
TVGL--DDPM-Forecast                     & 1.137 & 1.230 \\
DyGraphformer-T                         & 1.130 & 1.215 \\
\midrule
Training mean (no flow)          & 1.156 & 1.208 \\
Mean train displacement       & 1.607 & 1.686 \\
Persistence (copy last history)         & 1.607 & 1.685 \\
Linear history-drift                    & 2.786 & 2.889 \\
\midrule
Basic-CNN $\rightarrow$ TVGL            & 2.193 & 2.322 \\
LSTM $\rightarrow$ TVGL                 & 2.018 & 2.351 \\
cVAE $\rightarrow$ TVGL                 & 4.035 & 4.165 \\
cWGAN-GP $\rightarrow$ TVGL             & 2.138 & 2.269 \\
\bottomrule
\end{tabular}
\end{table}

\subsection{Chaotic Maps Forecasting}
\label{sec:exp-chaotic}

Table~\ref{tab:forecast_main} shows that the gains extend beyond EEG. Averaged over all nine system$\times$horizon settings (Lorenz, MacArthur, and Hopfield at $h\in\{3,5,7\}$, three seeds each), random walk TVGL--CFM performs best on both metrics, with mean AIRM $1.069$ and mean logE-RMSE $1.162$. This improves over TVGL--cVAE-Forecast ($1.129$/$1.188$), DyGraphformer-T ($1.130$/$1.215$), TVGL--DDPM-Forecast ($1.137$/$1.230$), persistence ($1.607$ AIRM), and the best raw-signal pipeline, cWGAN-GP $\rightarrow$ TVGL ($2.138$ AIRM). Full results are in Appendix~\ref{app:full-results}.

 Without class labels, \eqref{eq:classmean} reduces to a global training-set future mean, which reaches AIRM $1.156$; the mean-residual flow reaches $1.166$, while the history-informed coupling improves to $1.069$.

 At $h=5$, TVGL--CFM is best on Hopfield ($0.759$), TVGL--DDPM-Forecast on Lorenz ($1.453$ vs.\ $1.465$), and DyGraphformer-T on MacArthur ($1.019$ vs.\ $1.047$). Overall, TVGL--CFM is the strongest method across the chaotic benchmark.
\section{Related work}
\textsc{DiffeoCFM} \citep{collas2025} generates \emph{static} SPD covariance/correlation
matrices via pullback-geometry flow matching, while \citet{marti2020} use GANs for static
financial correlation matrices. More broadly, deep generative modelling spans VAEs
\citep{kingma2014}, GANs \citep{goodfellow2014}, normalizing flows
\citep{dinh2017,kingma2018glow}, autoregressive models \citep{vandenoord2016} and diffusion
\citep{ho2020,song2019,nichol2021,dhariwal2021,rombach2022}, with intrinsic SPD methods
surveyed by \citet{ju2025}.

We instead generate and forecast temporal trajectories of inferred graphs. Graph generative
models \citep{vignac2023,huang2025sfmg,williams2025weighted} primarily target topology,
often static and unweighted, rather than temporally coherent SPD precision trajectories.
Diffusion forecasters \citep{rasul2021} predict future values, whereas we forecast future
dependency structure. DyGraphformer \citep{han2025dygraphformer} infers time-varying
dependencies to improve value forecasting; our DyGraphformer-T baseline instead predicts the
dependency trajectory itself.

STFlow \citep{bennema2026stflow} conditions the source, while CGFM \citep{xu2025cgfm}
uses an auxiliary forecaster as the source so the flow models its residual structure. Our
random-walk coupling follows this principle using parameter-free extrapolation of recent
graph-state increments; the mean-centred variant instead keeps the Gaussian source and
recentres the target.

\section{Discussion and Conclusion}
We presented TVGL--CFM, a geometry-aware framework for generating and forecasting time-varying
SPD precision trajectories. While the components in our approach have already been established the novelty lies in the coherent combination of them to lift diffeomorphic flow matching from single SPD objects to
complete temporal trajectories on the product manifold $(\mathbb S_{++}^{p})^T$, and
learning one globally coupled vector field over the corresponding
$Td$-dimensional Euclidean trajectory space. Although the chart acts window-wise, the learned
dynamics do not: each temporal component can depend on the full trajectory through attention.
This distinguishes TVGL--CFM from a window-wise application of \textsc{DiffeoCFM}, which would
transport individual SPD matrices without learned temporal communication.

The same trajectory-level construction supports forecasting. We study two data-dependent
couplings: a mean-centred residual formulation that recentres the target around a training-set
future mean, and a random-walk formulation whose source extrapolates the observed prefix.
Across experiments, modelling the precision trajectory directly is substantially stronger than
generating raw signals and estimating TVGL afterwards, while remaining competitive with models
trained on the same TVGL representation. To our knowledge, this is the first extension of
Riemannian diffeomorphic flow matching to time series of matrices.

The main limitations arise from the target estimator and representation. TVGL hyperparameters
$\lambda$ and $\beta$ constrain the fidelity of the learned targets and are held fixed rather than
tuned per dataset, although Appendix~\ref{app:tvgl_sensitivity} shows no sharp degradation within
a factor-of-two neighbourhood of the settings we use. The embedding dimension also scales quadratically
with channel count \citep{oko2023}; Appendix~\ref{app:complexity_scaling} evaluates the pipeline
up to $p=512$, where dense TVGL and matrix-logarithm evaluation dominate the computational cost.
The framework is not tied to TVGL: Appendix~\ref{sec:loggle} demonstrates transfer to the LOGGLE
estimator \citep{yang2020loggle}, suggesting broader applicability to temporal graphical-model
estimators with compatible geometric structure.

Future work should explore alternative temporal estimators and manifolds,
data-dependent couplings such as optimal transport, and sensitivity to
temporal choices such as window size.

\appendix

\section{Extended background}
\label{app:extended-bg}

\subsection{Riemannian structures on SPD and correlation matrices}
The log-Euclidean metric of \eqref{eq:diffeo} is one of several Riemannian structures proposed
for SPD and correlation matrices, including the affine-invariant metric
\citep{skovgaard1984,pennec2006intrinsic}, the Cholesky-based geometries \citep{lin2019}, and
the stratified constructions of \citet{thanwerdas2022} and \citet{david2019}; it is the choice
under which flow matching reduces to a Euclidean problem. Such geometries also underpin
classical manifold statistics such as principal geodesic analysis \citep{fletcher2004}. Other
global charts on $\Spd$, such as the log-Cholesky map \citep{lin2019}, could be substituted
without changing the framework or any of the results of Appendix~\ref{app:pullback}.

\subsection{Generative modelling on manifolds}
Riemannian flow matching \citep{chen2024} sits within a broader effort to lift generative
models onto manifolds, including Riemannian score-based diffusion \citep{debortoli2022},
mixtures of Riemannian diffusion processes \citep{jo2023}, data-manifold metric flow matching
\citep{kapusniak2024}, and applications to materials \citep{miller2024}, proteins
\citep{huguet2024}, robot policies \citep{ding2025}, and Lie-group densities
\citep{falorsi2019}. For SPD data specifically, SPD-DDPM \citep{li2024spdddpm} and wrapped
Gaussians \citep{desurrel2025} operate intrinsically but at higher cost. Minibatch
optimal-transport couplings \citep{tong2024} can further straighten the conditional paths of
\eqref{eq:cfm}. A practical overview of the objective and its variants is given by
\citet{lipman2024guide}.

\subsection{Data-dependent couplings}
\label{app:couplings}
Standard flow matching pairs each data sample with an independent draw from a simple
source distribution. Choosing this pairing more carefully, for instance through
optimal-transport couplings \citep{tong2024}, which are grounded in displacement
interpolation \citep{mccann1997,villani2009}, can shorten transport and speed up training.
In conditional forecasting, however, the observed history already contains information about
the future. Recent trajectory-simulation work exploits this by constructing a source sample
from the observed prefix, for example a noisy continuation of the recent motion, so that the
flow transports from an informed prior rather than from uninformative Gaussian noise
\citep{bennema2026stflow}. A closely related idea in time-series forecasting is CGFM
\citep{xu2025cgfm}, which uses the predictive distribution of an auxiliary forecaster as the
source, so that flow matching learns the probabilistic structure of that forecaster's
residuals rather than treating them purely as an optimisation target. Both modify or condition
the \emph{source}, and our random-walk coupling of Section~\ref{sec:informedsources}
follows that principle, with the state a log-Euclidean precision graph rather than a particle
position and the source a parameter-free stochastic extrapolation of recent graph-state
increments in place of a separately trained auxiliary forecaster. Our mean-centred residual
coupling is a different construction: the source stays Gaussian and the target is recentred by
a training-set future mean. Since that mean has no learned parameters, running it forward with
a zero residual gives the zero-residual baseline reported alongside the full model.

\subsection{Time-varying network inference}
TVGL \citep{hallac2017} has been applied to domains as different as neuroscience, finance, and
automotive sensing. Its precursors on time-varying network inference include kernel smoothing
\citep{zhou2010}, $\ell_1$-fused penalties \citep{kolar2010}, and applications to
biological and brain networks \citep{ahmed2009,monti2014}; the underlying conditional-%
independence semantics are standard in the graphical-model literature
\citep{lauritzen1996,koller2009,rue2005}, and closely related structured estimators appear
in \citet{wytock2013,hallac2015,banerjee2008,yuan2007}.

\subsection{Motor-imagery decoding}
Covariance-based Riemannian representations and data-driven graph constructions provide
established ways to characterise multiscale and higher-order structure in EEG signals
\citep{barachant2011,kobler2022,roy2025,roy2026}. In motor imagery, discriminative
information is largely expressed through sensorimotor-rhythm (de)synchronisation
\citep{pfurtscheller1999}, commonly exploited by spatial filters such as CSP
\citep{blankertz2008,lotte2018}.

\section{Workflow overview and the two time variables}
\label{app:workflow}

\paragraph{The generative model, and its two time variables.}
The generator is a \emph{single transformer} that reads a whole precision trajectory as a
sequence of tokens (one token per window) and outputs the flow's velocity field; it is
defined in full in \eqref{eq:ufield} of Section~\ref{sec:flow}. It is entirely separate
from the TVGL stage: TVGL is a convex optimisation that produces the target
trajectories, whereas the transformer is the neural model trained on those targets.
This model reads two distinct notions of time. \emph{Flow time} $s\in[0,1]$ is the CFM
transport variable: $s=0$ is the source, $s=1$ is the target, and sampling integrates the ODE
in $s$. \emph{Window time} $\tau_i:=i/(T-1)$ indexes the fixed TVGL windows $i=1,\ldots,T$ and
tells the transformer where each token lies along the trajectory. The model receives both
through separate Fourier features $\mathrm{FF}_s$ and $\mathrm{FF}_\tau$ (see
\eqref{eq:ufield}). When needed, we write $z_{s,i}$ for flow time $s$ at window $i$; in
particular, $\hat z_{1,i}$ is the predicted clean endpoint at window $i$.

Figure~\ref{fig:tvglcfm_generation_forecasting} summarises the two tasks that share the TVGL
target-construction and log-Euclidean embedding stages of Section~\ref{sec:method}:
class-conditional generation of a whole precision trajectory, and history-conditioned
forecasting of its continuation.

\begin{figure}[h]
    \centering
    \includegraphics[width=\textwidth]{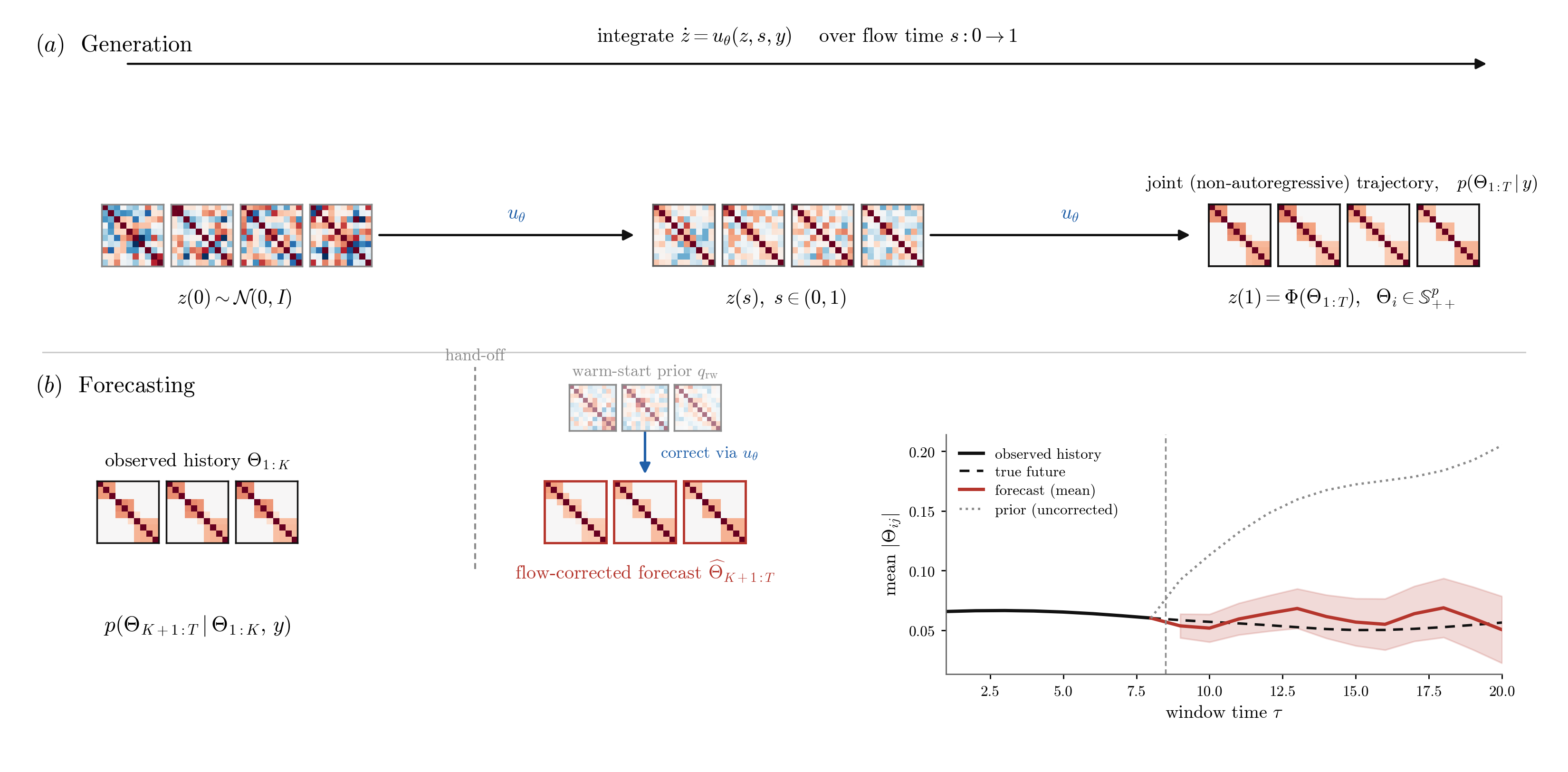}
    \caption{
    \textbf{TVGL-CFM generation and forecasting workflows.}
    \textbf{(a) Generation.} A class-conditioned conditional flow matching model transports a simple source trajectory
    \(z(0)\sim\mathcal{N}(0,I)\) to a joint TVGL precision trajectory by integrating the learned vector field
    \(\dot z = u_\theta(z,s,y)\) over flow time \(s\in[0,1]\). The terminal state is decoded as
    \(z(1)=\Phi(\Theta_{1:T})\), yielding a non-autoregressive sample from the conditional trajectory distribution
    \(p(\Theta_{1:T}\mid y)\), where each \(\Theta_i\in\mathbb{S}^{p}_{++}\).
    \textbf{(b) Forecasting.} Given an observed TVGL history \(\Theta_{1:K}\), the model integrates the learned
    vector field under one of two data-dependent couplings, either Gaussian residuals around the
    training-mean future \(m^{\mathrm f}\) or a random-walk random-walk source, and decodes to
    obtain a forecast \(\widehat{\Theta}_{K+1:T}\). The right
    panel illustrates the resulting dynamic forecast summary: observed history, true future, source trajectory,
    and the TVGL-CFM predictive mean with uncertainty over the forecast horizon. The vertical dashed line marks
    the handoff between observed and predicted windows.
    }
    \label{fig:tvglcfm_generation_forecasting}
\end{figure}

\section{TVGL target construction and solver}
\label{app:tvgl-solver}
This appendix expands Section~\ref{sec:tvgl} of the main paper: the shrinkage covariance
estimator, the batched ADMM solver for \eqref{eq:tvgl}, and the relationship between SPD
validity and sparsity in the decoded output.

\subsection{Windowed shrinkage covariance}
Each sequence's $p$-channel signal is $z$-scored per channel over time and split into
$T$ contiguous windows of length $n$. Because $n$ may be comparable to $p$, the raw
sample covariance is poorly conditioned; we use the oracle-approximating-shrinkage (OAS)
estimator \citep{chen2010}, which for a window with sample covariance $S$ returns
\begin{equation}
\widehat S=(1-\rho)\,S+\rho\,\hat\mu I,
\quad \hat\mu=\tfrac{1}{p}\tr S,
\quad \rho=\min\!\Big(1,\ \frac{\alpha+\hat\mu^2}{(n+1)\big(\alpha-\hat\mu^2/p\big)}\Big),
\quad \alpha=\tfrac{1}{p^2}\tr(S^2).
\label{eq:oas}
\end{equation}
All windows are computed in a single vectorised pass over (sequence, window). We divide
every $\widehat S$ by the mean training diagonal so that covariances are $O(1)$ and the
TVGL $\lambda$ is interpretable.

\subsection{Batched TVGL targets}
We solve \eqref{eq:tvgl} for all sequences and windows simultaneously with ADMM
\citep{boyd2011,boyd2004}. Introducing consensus copies $Z_0,Z_1,Z_2$ for the sparsity and
the two temporal roles of each slice, the updates have closed forms in terms of proximal
operators \citep{parikh2014}. Consider a single window (we suppress the window index $i$).
The likelihood (proximal) update first forms the symmetric matrix
$M=\tfrac12(A+A^\top)-\eta S$, where $A$ is the averaged combination of the consensus and
dual variables for that slice (Algorithm~\ref{alg:tvgl}, line~3), $S$ is its shrinkage
covariance, and $\eta=n/(\rho c_i)>0$ is the per-slice weight defined on line~1 ($n$
samples per window, penalty parameter $\rho$, and $c_i$ the number of active roles of slice
$i$). Taking the eigendecomposition $M=Q\diag(d)Q^\top$, with orthogonal eigenvector matrix
$Q$ and eigenvalue vector $d\in\R^{p}$, the update has the closed form
\begin{equation}
\Th \;=\; Q\,\diag\!\Big(\tfrac{1}{2}\big(d+\sqrt{d^2+4\eta}\big)\Big)\,Q^\top
\;\succ\;0,
\label{eq:thupdate}
\end{equation}
which is guaranteed SPD because every transformed eigenvalue
$\tfrac12\big(d_j+\sqrt{d_j^2+4\eta}\big)$ is strictly positive. The sparsity update is the
off-diagonal soft-threshold $\softthr_{\lambda/\rho}$ (diagonal preserved). The temporal
update is the proximal operator of the same evolutionary penalty $\psi$ chosen in
\eqref{eq:tvgl} (Section~\ref{sec:tvgl-bg}), applied to the consensus difference; for the
three choices of $\psi$ it is
\begin{equation}
\mathrm{prox}_{\kappa\psi}(D)=
\begin{cases}
\softthr_{\kappa}(D), & \psi=\ell_1\ \text{(elementwise)},\\[2pt]
D/(1+2\kappa), & \psi=\ell_2^2\ \text{(Laplacian)},\\[2pt]
\big[\,\max(0,1-\kappa/\norm{[D]_j}_2)\,[D]_j\,\big]_j, & \psi=\text{group }\ell_2\ \text{(per column $j$)}.
\end{cases}
\label{eq:prox}
\end{equation}
Algorithm~\ref{alg:tvgl} states the procedure; the output $\Th_{n,i}$ is a sparse SPD
precision trajectory per sequence.

\begin{algorithm}[H]
\caption{Batched TVGL via ADMM (all sequences/windows)}
\label{alg:tvgl}
\begin{algorithmic}[1]
\Require shrinkage covariances $S_{n,i}$; $\lambda,\beta,\rho$; samples/window $n$; iterations $K_{\mathrm{admm}}$
\State init $\Th,Z_0,Z_1,Z_2\!\leftarrow\! I$;\ \ $U_0,U_1,U_2\!\leftarrow\! 0$;\ \ $c_i\!\leftarrow\!1+\![i<T]\!+\![i>0]$;\ \ $\eta_i\!\leftarrow\! n/(\rho c_i)$
\For{$k=1$ to $K_{\mathrm{admm}}$}
  \State $A_i \leftarrow \big(Z_{0,i}-U_{0,i} + [i{<}T](Z_{1,i}-U_{1,i}) + [i{>}1](Z_{2,i}-U_{2,i})\big)/c_i$
  \State $\Th_i \leftarrow$ Eq.~\eqref{eq:thupdate} with $M_i=\tfrac12(A_i+A_i^\top)-\eta_i S_i$ \Comment{batched eigendecomposition}
  \State $Z_{0,i} \leftarrow \softthr_{\lambda/\rho}(\Th_i+U_{0,i})$ off-diagonal; diagonal $\leftarrow (\Th_i+U_{0,i})$
  \State $E_i \leftarrow \mathrm{prox}_{(2\beta/\rho)\psi}\!\big((\Th_{i}+U_{2,i})-(\Th_{i-1}+U_{1,i-1})\big)$ \Comment{Eq.~\eqref{eq:prox}}
  \State $Z_{1,i-1}\!\leftarrow\!\tfrac12(\Sigma_i-E_i)$,\ \ $Z_{2,i}\!\leftarrow\!\tfrac12(\Sigma_i+E_i)$ with $\Sigma_i=(\Th_{i-1}{+}U_{1,i-1}){+}(\Th_i{+}U_{2,i})$
  \State $U_{0}\!\mathrel{+}=\!\Th-Z_0$;\ \ $U_{1}\!\mathrel{+}=\!\Th-Z_1$;\ \ $U_{2}\!\mathrel{+}=\!\Th-Z_2$ \Comment{on active slices}
\EndFor
\State \Return $\tfrac12(\Th+\Th^\top)$
\end{algorithmic}
\end{algorithm}

\subsection{SPD validity and exact sparse readout}
\label{subsec:spd-sparsity}

\begin{figure}[H]
    \centering
    \includegraphics[width=\textwidth]{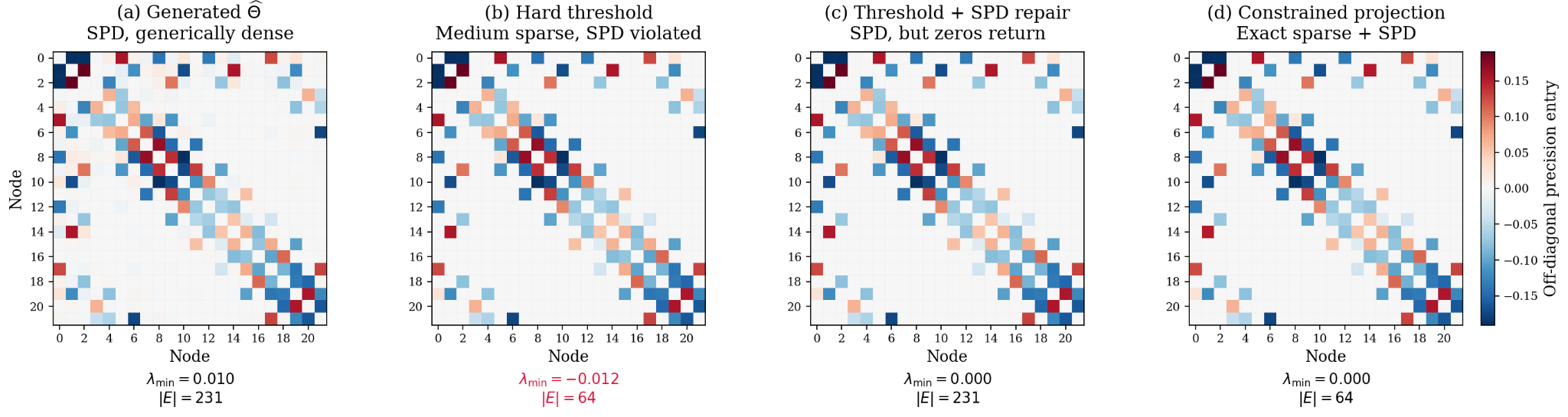}
    \caption{
    Illustration of SPD validity versus exact sparsity for a TVGL--CFM generated precision matrix.
    The decoded output is symmetric positive-definite by construction but generically dense.
    Hard thresholding yields an exact sparse support but can violate positive-definiteness.
    A generic eigenvalue-based SPD repair restores positive-definiteness but reintroduces
    nonzero entries on the thresholded support complement.
    The constrained projection instead computes the nearest matrix satisfying both the
    prescribed zero pattern and $\Theta \succeq \varepsilon I$, yielding an exactly sparse
    SPD readout.
    }
    \label{fig:sparse_spd_readout}
\end{figure}

Decoding through $\phi^{-1}$ guarantees that every generated
$\widehat{\Theta}_i$ is symmetric positive-definite, since the inverse
log-Euclidean map is the matrix exponential. This guarantees manifold validity,
but not exact sparsity: the exponential of a generic symmetric matrix is dense,
so entries that are structurally zero in the sparse TVGL targets are typically
represented by small but nonzero values in the generated precision matrices.

\label{rem:spd-not-sparse}%
We therefore distinguish between the sparse TVGL targets used for training and
oracle evaluation, the dense SPD trajectories produced directly by the model,
and an optional sparse SPD readout. All quantitative results are computed on the
direct model output, without post-processing. The sparse readout is intended for
applications where an exact graph support is required.

A naive threshold-and-project procedure is insufficient for this purpose.
Thresholding small off-diagonal entries can destroy positive-definiteness, while
a subsequent unconstrained projection back onto $\Spd$ can reintroduce nonzero
values at entries that were explicitly thresholded to zero. Instead, we first
extract a candidate support from the generated precision matrix. For a threshold
$\tau>0$, define
\begin{equation}
\mathcal{S}_{\tau}(\widehat{\Theta}_i)
=
\left\{
(a,b):
a=b
\ \text{or}\
\left|\widehat{\Theta}_{i,ab}\right|\geq \tau
\right\}.
\label{eq:sparse-support}
\end{equation}
We then compute the nearest matrix that simultaneously respects this support
and remains strictly positive-definite:
\begin{equation}
\Theta_i^{\mathrm{sp}}
=
\arg\min_{\Theta\in\mathbb{S}^p}
\frac{1}{2}
\left\|
\Theta-\widehat{\Theta}_i
\right\|_F^2
\label{eq:nearest-sparse-spd}
\end{equation}
subject to
\begin{equation}
\Theta \succeq \varepsilon I,
\qquad
\Theta_{ab}=0
\quad
\forall (a,b)\notin
\mathcal{S}_{\tau}(\widehat{\Theta}_i),
\label{eq:nearest-sparse-spd-constraints}
\end{equation}
where $\varepsilon>0$ is a small eigenvalue margin.

This is a convex optimization problem. Importantly,
the retained nonzero entries are not frozen to their values in
$\widehat{\Theta}_i$; they may move slightly so that positive-definiteness can
be restored while remaining as close as possible to the generated matrix.
Consequently, the solution satisfies both
\begin{equation}
\lambda_{\min}\!\left(\Theta_i^{\mathrm{sp}}\right)
\geq \varepsilon
\end{equation}
and
\begin{equation}
\Theta_{i,ab}^{\mathrm{sp}}=0
\quad
\forall (a,b)\notin
\mathcal{S}_{\tau}(\widehat{\Theta}_i),
\end{equation}
so the resulting readout is simultaneously SPD and exactly sparse.

For consistency with the TVGL estimator, we may choose
$\tau=\lambda/\rho$, matching the cutoff appearing in the TVGL sparsity
proximal update. Alternatively, $\tau$ can be treated as a visualization or
support-selection parameter. This operation is applied only after generation
and does not alter the learned flow or the reported quantitative metrics.
Thus, the model output itself remains the dense SPD trajectory defined by the
manifold decoder, while the constrained projection provides an optional exact
sparse graph representation when one is required.

\section{Training and sampling algorithms}
\label{app:algorithms}
Algorithms~\ref{alg:train} and \ref{alg:sample} give training and sampling for the
class-conditional generator of Section~\ref{sec:flow}, and Algorithm~\ref{alg:forecast} gives
block, non-autoregressive forecasting under either coupling of
Section~\ref{sec:informedsources}.

\begin{algorithm}[H]
\caption{TVGL--CFM generation training}
\label{alg:train}
\begin{algorithmic}[1]
\Require standardised target trajectories $z_n$, labels $y_n$; window-time grid $\tau_{1:T}$, $\tau_i=i/(T-1)$; weight $\lambda_{\mathrm{temp}}$
\While{not converged}
  \State sample batch $(z_1,y)$; $z_0\sim\mathcal N(0,I)$;\quad flow time $s\sim\mathcal U[0,1]$
  \State $z_s\leftarrow(1-s)z_0+s\,z_1$;\quad $v\leftarrow u_\theta(z_s,s,y,\tau_{1:T})$;\quad $\hat z_1\leftarrow z_s+(1-s)v$
  \State $\mathcal L\leftarrow \norm{v-(z_1-z_0)}_2^2+\lambda_{\mathrm{temp}}\,\tfrac{1}{T-1}\sum_{i=2}^{T} \psi_E(\hat z_{1,i}-\hat z_{1,i-1})$
  \State update $\theta$ by AdamW on $\mathcal L$
\EndWhile
\end{algorithmic}
\end{algorithm}

\begin{algorithm}[H]
\caption{TVGL--CFM generation sampling}
\label{alg:sample}
\begin{algorithmic}[1]
\Require label $y$; flow-time steps $L$; window-time grid $\tau_{1:T}$; standardiser $(\mu,\sigma)$
\State $z\sim\mathcal N(0,I)$;\quad $h\leftarrow 1/L$ \Comment{flow-time step}
\For{$\ell=0$ to $L-1$} \Comment{integrate ODE in flow time at fixed $\tau_{1:T}$}
  \State $z\leftarrow \mathrm{RK4\text{-}step}\big(u_\theta(\cdot,\cdot,y,\tau_{1:T}),\,z,\,\ell h,\,h\big)$
\EndFor
\State $z\leftarrow \sigma\odot z+\mu$;\quad $\widehat\Th\leftarrow \phi^{-1}(z)$ \Comment{matrix exponential per window}
\State \textbf{optional} (for $\ell_1$): off-diagonal soft-threshold $\widehat\Th$, then project to $\Spd$
\State \Return $\widehat\Th$
\end{algorithmic}
\end{algorithm}

\begin{algorithm}[H]
\caption{TVGL--CFM forecasting under either data-dependent coupling (block, non-autoregressive)}
\label{alg:forecast}
\begin{algorithmic}[1]
\Require history $z_{1:K}$, optional label $y$; coupling $\in\{$mean-residual, random-walk$\}$; training mean $\mf$ (per class where labels exist) from Eq.~\eqref{eq:classmean}; flow-time steps $L$; ensemble size $M$; window-time grid $\tau_{K+1:T}$; standardiser $(\mu,\sigma)$
\State $H\leftarrow\mathcal E_\theta(z_{1:K})\in\R^{K\times d_{\mathrm{model}}}$ \Comment{all $K$ history tokens, Eq.~\eqref{eq:histenc}}
\If{coupling $=$ random-walk} \State estimate $\widehat\mu,\widehat\sigma$ from recent history increments by Eq.~\eqref{eq:rwstats} \EndIf
\For{$m=1$ to $M$}
  \If{coupling $=$ mean-residual}
    \State $z^{\mathrm f}\sim\mathcal N(0,I)\in\Ef$ \Comment{residual coordinates}
  \Else
    \State $z^{\mathrm f}\sim\qrw(\cdot\mid z_{1:K})$ by Eq.~\eqref{eq:rwprior} \Comment{absolute coordinates}
  \EndIf
  \State $h\leftarrow 1/L$
  \For{$\ell=0$ to $L-1$} \Comment{integrate ODE in flow time at fixed $\tau_{K+1:T}$}
    \State $z^{\mathrm f}\leftarrow \mathrm{RK4\text{-}step}\big(u_\theta(\cdot,\cdot,y,\tau_{K+1:T},H),\,z^{\mathrm f},\,\ell h,\,h\big)$
  \EndFor
  \State $z\leftarrow \mf+z^{\mathrm f}$ if mean-residual, else $z\leftarrow z^{\mathrm f}$
  \State $\widehat\Th^{(m)}_{K+1:T}\leftarrow \phi^{-1}(\sigma\odot z+\mu)$
\EndFor
\State \Return ensemble $\{\widehat\Th^{(m)}_{K+1:T}\}_{m=1}^{M}$ (point forecast: mean in embedding; spread: uncertainty)
\end{algorithmic}
\end{algorithm}

\section{Pullback Geometry for Time-Varying Precision Trajectories}
\label{app:pullback}
This appendix lifts the pullback-CFM framework of \citet{collas2025}
from a single SPD matrix to a trajectory of SPD precision matrices, adds
embedding standardisation, and establishes the equivalences on which the
generation and forecasting models of Section~\ref{sec:method} both rely, including the
data-dependent source couplings of Section~\ref{sec:informedsources}. All results are stated for the log-Euclidean
diffeomorphism of Eq.~\eqref{eq:diffeo}, but extend verbatim to any
global diffeomorphism $\Spd\to\R^{d}$ (e.g.\ log-Cholesky).

\subsection{Setup and notation}
\label{app:notation}

Let $p$ be the channel count, $d:=p(p+1)/2$, and $T$ the trajectory length.
Recall the log-Euclidean diffeomorphism from Eq.~\eqref{eq:diffeo}:
\begin{equation}
\label{eq:phi-def}
  \phi:\Spd\to\R^{d},\quad \phi(\Theta)=\veclt(\log\Theta),
  \qquad
  \phi^{-1}:\R^{d}\to\Spd,\quad \phi^{-1}(\eta)=\exp\bigl(\veclt^{-1}(\eta)\bigr),
\end{equation}
where $\veclt$ vectorises the lower-triangular part with off-diagonal
entries scaled by $\sqrt 2$ so that
\begin{equation}
\label{eq:veclt-isometry}
  \|\phi(\Theta)-\phi(\Theta')\|_{2}
  \;=\;\|\log\Theta-\log\Theta'\|_{F}
  \;=\;\dLE(\Theta,\Theta'),
\end{equation}
i.e.\ $\veclt$ is an isometry between the symmetric matrices
$(\Sym,\langle\cdot,\cdot\rangle_{F})$ and
$(\R^{d},\langle\cdot,\cdot\rangle_{2})$.

\paragraph{Product manifold.}
The trajectory manifold is
\begin{equation}
\label{eq:Mtraj-def}
  \Mtraj
  \;:=\;\underbrace{\Spd\times\cdots\times\Spd}_{T\text{ factors}},
\end{equation}
with tangent space $T_{\Theta_{1:T}}\Mtraj=\prod_{i=1}^{T}T_{\Theta_{i}}\Spd$.
We equip $\Mtraj$ with the \emph{product log-Euclidean metric}
\begin{equation}
\label{eq:product-le-metric}
  \langle\xi,\eta\rangle^{\mathrm{prod}}_{\Theta_{1:T}}
  \;:=\;\sum_{i=1}^{T}\langle\xi_{i},\eta_{i}\rangle^{\mathrm{LE}}_{\Theta_{i}}.
\end{equation}
Let $E:=\R^{T\times d}$ with the standard Euclidean inner product, and define
\begin{equation}
\label{eq:Phi-def}
  \Phi:\Mtraj\to E,
  \qquad
  \Phi(\Theta_{1:T}):=\bigl(\phi(\Theta_{1}),\dots,\phi(\Theta_{T})\bigr).
\end{equation}
\paragraph{Identities used throughout.}
Two identities of differentials, which follow from the chain rule applied to
$\Phi\circ\Phi^{-1}=\mathrm{Id}_{E}$ and $\Phi^{-1}\circ\Phi=\mathrm{Id}_{\Mtraj}$, are
\begin{equation}
\label{eq:dphi-inverse}
  D\Phi\bigl(\Phi^{-1}(z)\bigr)\circ D\Phi^{-1}(z)=\mathrm{Id}_{E},
  \qquad
  D\Phi^{-1}\bigl(\Phi(x)\bigr)\circ D\Phi(x)=\mathrm{Id}_{T_{x}\Mtraj}.
\end{equation}

\subsection{Proof of Proposition~\ref{prop:prod-diffeo} (product diffeomorphism and pullback metric)}
\label{app:proof-prod-diffeo}

\begin{proof}[Proof of Proposition~\ref{prop:prod-diffeo}]
\emph{Diffeomorphism.}
$\phi$ is a $C^{\infty}$ bijection $\Spd\to\R^{d}$ with $C^{\infty}$ inverse, since
the matrix logarithm and exponential are smooth on $\Spd$ and on the space
$\Sym$ of symmetric matrices respectively, and $\veclt$ is a linear
isomorphism. The map
\[
  \Phi=\phi\times\cdots\times\phi:\,(\Spd)^{T}\longrightarrow(\R^{d})^{T}\cong\R^{T\times d}
\]
is therefore a $C^{\infty}$ bijection with $C^{\infty}$ inverse
$\Phi^{-1}=\phi^{-1}\times\cdots\times\phi^{-1}$, hence a global diffeomorphism.

\emph{Pullback metric.}
The differential $D\Phi(\Theta_{1:T})$ is block-diagonal in the trajectory
index: for $\xi\in T_{\Theta_{1:T}}\Mtraj$,
\begin{equation}
\label{eq:dphi-block}
  D\Phi(\Theta_{1:T})[\xi]
  \;=\;\bigl(D\phi(\Theta_{1})[\xi_{1}],\dots,D\phi(\Theta_{T})[\xi_{T}]\bigr)\in E.
\end{equation}
Writing $g_{E}(u,v)=\sum_{i,k}u_{i,k}v_{i,k}=\sum_{i}\langle u_{i},v_{i}\rangle_{2}$
and applying \eqref{eq:dphi-block} to $\xi,\eta\in T_{\Theta_{1:T}}\Mtraj$,
\[
\begin{aligned}
  (\Phi^{*}g_{E})_{\Theta_{1:T}}(\xi,\eta)
  &= g_{E}\bigl(D\Phi(\Theta_{1:T})[\xi],\,D\Phi(\Theta_{1:T})[\eta]\bigr)\\
  &= \sum_{i=1}^{T}\bigl\langle D\phi(\Theta_{i})[\xi_{i}],\,D\phi(\Theta_{i})[\eta_{i}]\bigr\rangle_{2}\\
  &= \sum_{i=1}^{T}(\phi^{*}g_{E})_{\Theta_{i}}(\xi_{i},\eta_{i})
   = \sum_{i=1}^{T}\langle\xi_{i},\eta_{i}\rangle^{\mathrm{LE}}_{\Theta_{i}},
\end{aligned}
\]
where the third equality is the definition of the pullback metric on each
factor and the fourth uses that $\phi^{*}g_{E}$ is the log-Euclidean
metric on $\Spd$ \citep[Sec.~3.2]{collas2025}. The right-hand side
is the product log-Euclidean metric \eqref{eq:product-le-metric}.
\end{proof}

\subsection{Proof of Proposition~\ref{prop:traj-cfm} (trajectory CFM loss equivalence)}
\label{app:proof-traj-cfm}

Proposition~\ref{prop:traj-cfm} is stated in Section~\ref{sec:flow}; the
pullback field \eqref{eq:pullback-field-traj} and target loss
\eqref{eq:traj-cfm-loss} are as defined there.

\begin{proof}[Proof of Proposition~\ref{prop:traj-cfm}]
Fix a label $y$ and draw $\Theta_{1:T}^{0}\sim p(\cdot\mid y)$,
$\Theta_{1:T}^{1}\sim q(\cdot\mid y)$. Set $z_{0}:=\Phi(\Theta_{1:T}^{0})$,
$z_{1}:=\Phi(\Theta_{1:T}^{1})$, and for $s\in[0,1]$,
\begin{equation}
\label{eq:zt-def}
  z_{s}:=(1-s)z_{0}+s\,z_{1}\in E,
  \qquad
  \gamma(s):=\Phi^{-1}(z_{s})\in\Mtraj.
\end{equation}

\emph{$\gamma$ is the geodesic on $(\Mtraj,\Phi^{*}g_{E})$.}
On each factor, $\gamma_{i}(s)=\phi^{-1}((1-s)\phi(\Theta_{i}^{0})+s\,\phi(\Theta_{i}^{1}))$
is the log-Euclidean geodesic from $\Theta_{i}^{0}$ to $\Theta_{i}^{1}$
(\citealp[Eq.~14]{collas2025} applied to $\phi$). Geodesics on a
product Riemannian manifold equipped with the product metric decouple into
component-wise geodesics, so $\gamma$ is the geodesic on
$(\Mtraj,\Phi^{*}g_{E})$ connecting $\Theta_{1:T}^{0}$ to $\Theta_{1:T}^{1}$.

\emph{Velocity.}
The straight line $z_{s}$ in $E$ has constant derivative $\dot z_{s}=z_{1}-z_{0}$.
Applying the chain rule to $\gamma=\Phi^{-1}\circ z$,
\begin{equation}
\label{eq:gamma-dot}
  \dot\gamma(s)
  \;=\;D\Phi^{-1}(z_{s})\bigl[\dot z_{s}\bigr]
  \;=\;D\Phi^{-1}(z_{s})\bigl[z_{1}-z_{0}\bigr].
\end{equation}

At every $s$, the pullback metric satisfies
\begin{equation}
\label{eq:pullback-norm}
  \|\zeta\|_{\gamma(s)}^{2}
  \;=\;(\Phi^{*}g_{E})_{\gamma(s)}(\zeta,\zeta)
  \;=\;\bigl\|D\Phi(\gamma(s))[\zeta]\bigr\|_{E}^{2}
  \quad\text{for all }\zeta\in T_{\gamma(s)}\Mtraj.
\end{equation}
Set $\zeta:=u^{\Mtraj}_{\theta}(s,\gamma(s),y)-\dot\gamma(s)$. By linearity of
the differential,
\[
  D\Phi(\gamma(s))[\zeta]
  =D\Phi(\gamma(s))\bigl[u^{\Mtraj}_{\theta}(s,\gamma(s),y)\bigr]
   -D\Phi(\gamma(s))[\dot\gamma(s)].
\]
The first term is $u^{E}_{\theta}(s,z_{s},y)$ by
\eqref{eq:pullback-field-traj} (using $\gamma(s)=\Phi^{-1}(z_{s})$). For the
second term, \eqref{eq:gamma-dot} gives
$D\Phi(\gamma(s))[\dot\gamma(s)]
=D\Phi(\gamma(s))\circ D\Phi^{-1}(z_{s})\bigl[z_{1}-z_{0}\bigr]
=z_{1}-z_{0}$,
where the last equality is the left identity in \eqref{eq:dphi-inverse}.
Combining,
\begin{equation}
\label{eq:integrand-id}
  D\Phi(\gamma(s))\bigl[u^{\Mtraj}_{\theta}(s,\gamma(s),y)-\dot\gamma(s)\bigr]
  \;=\;u^{E}_{\theta}(s,z_{s},y)-(z_{1}-z_{0}).
\end{equation}
Substituting into \eqref{eq:pullback-norm}:
\begin{equation}
\label{eq:pointwise-cfm}
  \bigl\|u^{\Mtraj}_{\theta}(s,\gamma(s),y)-\dot\gamma(s)\bigr\|_{\gamma(s)}^{2}
  \;=\;\bigl\|u^{E}_{\theta}(s,z_{s},y)-(z_{1}-z_{0})\bigr\|_{E}^{2}.
\end{equation}

\emph{Expectation.}
\eqref{eq:pointwise-cfm} holds pointwise for every realisation of
$(s,y,\Theta_{1:T}^{0},\Theta_{1:T}^{1})$. Taking expectation under
$s\sim\mathcal{U}[0,1]$, $y\sim\pi_{Y}$, $\Theta_{1:T}^{0}\sim p(\cdot\mid y)$,
$\Theta_{1:T}^{1}\sim q(\cdot\mid y)$, and using the substitution
$z_{0}=\Phi(\Theta_{1:T}^{0})$, $z_{1}=\Phi(\Theta_{1:T}^{1})$
(whose marginal laws are $\Phi_{\#}p(\cdot\mid y)$ and $\Phi_{\#}q(\cdot\mid y)$
respectively) yields \eqref{eq:traj-cfm-loss}.
\end{proof}

\subsection{Proof of Proposition~\ref{prop:traj-ode} (trajectory ODE equivalence)}
\label{app:proof-traj-ode}

\begin{proof}[Proof of Proposition~\ref{prop:traj-ode}]
\emph{Time derivative.}
By the chain rule applied to $\Theta_{1:T}=\Phi^{-1}\circ z$,
\[
  \dot\Theta_{1:T}(s)
  \;=\;D\Phi^{-1}(z(s))[\dot z(s)]
  \;=\;D\Phi^{-1}(z(s))\bigl[u^{E}_{\theta}(s,z(s),y)\bigr],
\]
where the second equality uses the ODE for $z$. Substituting
$z(s)=\Phi(\Theta_{1:T}(s))$ and using the definition
\eqref{eq:pullback-field-traj},
\[
  u^{E}_{\theta}(s,\Phi(\Theta_{1:T}(s)),y)
  =D\Phi(\Theta_{1:T}(s))\bigl[u^{\Mtraj}_{\theta}(s,\Theta_{1:T}(s),y)\bigr].
\]
Therefore
\[
  \dot\Theta_{1:T}(s)
  =D\Phi^{-1}\bigl(\Phi(\Theta_{1:T}(s))\bigr)\circ D\Phi(\Theta_{1:T}(s))\bigl[u^{\Mtraj}_{\theta}(s,\Theta_{1:T}(s),y)\bigr]
  =u^{\Mtraj}_{\theta}(s,\Theta_{1:T}(s),y),
\]
by the right identity in \eqref{eq:dphi-inverse}.

\emph{Initial condition.}
$\Theta_{1:T}(0)=\Phi^{-1}(z(0))=\Phi^{-1}(\Phi(\Theta_{1:T}^{0}))=\Theta_{1:T}^{0}$.

\emph{Slice-wise SPD.}
For each $i$, $\Theta_{i}(s)=\phi^{-1}(z_{i}(s))=\exp(\veclt^{-1}(z_{i}(s)))$.
Since $\veclt^{-1}:\R^{d}\to\Sym$ is a linear isomorphism onto the
symmetric matrices, $\veclt^{-1}(z_{i}(s))\in\Sym$. The matrix
exponential restricts to a bijection $\exp:\Sym\to\Spd$
(diagonalise: $S=Q\Lambda Q^{\top}$ implies $\exp(S)=Q\exp(\Lambda)Q^{\top}$
with positive eigenvalues), so $\Theta_{i}(s)\in\Spd$ for every $i,s$.
\end{proof}

\subsection{Proof of Proposition~\ref{prop:traj-rk} (trajectory Runge--Kutta equivalence)}
\label{app:proof-traj-rk}

Proposition~\ref{prop:traj-rk} is stated in Section~\ref{sec:flow}; the scheme
has $\nu$ stages with Butcher tableau $(a_{ij},b_{i},c_{i})$ and step size $h$.

\begin{proof}[Proof of Proposition~\ref{prop:traj-rk}]
\emph{Block-diagonal structure of all manifold operations.}
On $(\Mtraj,\Phi^{*}g_{E})$ all the geometric operations used by an explicit
Riemannian Runge--Kutta scheme, namely the differential $D\Phi$ (already
established as block-diagonal in \eqref{eq:dphi-block}), its inverse
$D\Phi^{-1}$, the exponential map
$\expmap_{\Theta_{1:T}}(\zeta)=\Phi^{-1}(\Phi(\Theta_{1:T})+D\Phi(\Theta_{1:T})[\zeta])$,
and parallel transport
$\mathrm{PT}_{x\to y}(\eta)=(D\Phi(y))^{-1}(D\Phi(x)[\eta])$,
all decompose into the corresponding operations on each factor $\Spd$.
This is because the product metric is the sum of factor-wise metrics, the
diffeomorphism is the Cartesian product of factor diffeomorphisms, and the
exponential/parallel-transport formulas
\citep[Eqs.~15--16]{collas2025} are determined pointwise by $D\Phi$.

\emph{Reduction to the single-matrix case.}
The Riemannian Runge--Kutta scheme on $\Mtraj$ at step $\ell$ evaluates
$u^{\Mtraj}_{\theta}$ at $\nu$ intermediate points
$\{x_{j}^{\ell}\}_{j=1}^{\nu}\subset\Mtraj$, parallel-transports each tangent
vector back to $T_{\Theta_{1:T}^{\ell}}\Mtraj$, takes a linear combination with
weights $b_{i}$, and applies the exponential map at $\Theta_{1:T}^{\ell}$. By
the block-diagonal structure above, this procedure executes
independently and simultaneously on each of the $T$ factors $\Spd$.
On factor $i$, the iterates are exactly those produced by the same scheme
applied to the single-matrix Riemannian ODE
$\dot\Theta_{i}=u^{\Spd}_{\theta,i}$ with initial condition $\Theta_{i}^{0}$,
where $u^{\Spd}_{\theta,i}$ is the $i$-th block of $u^{\Mtraj}_{\theta}$.
Proposition~3 of \citet{collas2025} applied to this $i$-th factor
yields
\[
  \Theta_{i}^{\ell}=\phi^{-1}(z_{i}^{\ell})
  \quad\text{for all }\ell\in\N,
\]
where $\{z_{i}^{\ell}\}$ is the corresponding Euclidean RK trajectory on
the $i$-th block. Stacking across $i=1,\dots,T$ gives
\eqref{eq:rk-equivalence-traj}.

\emph{Slice-wise SPD.}
Slice-wise SPD validity follows from
$\Theta_{i}^{\ell}=\phi^{-1}(z_{i}^{\ell})\in\Spd$ for every $\ell$, as in
the proof of Proposition~\ref{prop:traj-ode}.
\end{proof}

\subsection{Proof of Proposition~\ref{prop:standardisation} (standardised diffeomorphism)}
\label{app:proof-standardisation}

In Section~\ref{sec:embedding}, each log-Euclidean coordinate is standardised
using training statistics $\mu\in\R^{d}$ and $\sigma\in\R^{d}_{>0}$ broadcast
across the trajectory index. Define
\begin{equation}
\label{eq:Psi-def}
  \Psi:E\to E,\qquad
  \Psi(z)_{i,k}:=\frac{z_{i,k}-\mu_{k}}{\sigma_{k}},
  \quad
  \Psi^{-1}(\tilde z)_{i,k}=\mu_{k}+\sigma_{k}\tilde z_{i,k}.
\end{equation}

\begin{proof}[Proof of Proposition~\ref{prop:standardisation}]
\emph{Diffeomorphism.}
$\Psi$ is an affine bijection of $E$ with constant Jacobian
$D\Psi(z)=\diag(\sigma_{1}^{-1},\dots,\sigma_{d}^{-1})$ broadcast across the
trajectory index $i$; since $\sigma_{k}>0$, both $\Psi$ and $\Psi^{-1}$ are
smooth. A composition of two global diffeomorphisms is a global
diffeomorphism, so $\tildePhi=\Psi\circ\Phi$ is one.

\emph{Pullback metric.}
By the chain rule,
$D\tildePhi(\Theta_{1:T})=D\Psi(\Phi(\Theta_{1:T}))\circ D\Phi(\Theta_{1:T})$.
For $\xi,\eta\in T_{\Theta_{1:T}}\Mtraj$,
\[
\begin{aligned}
  (\tildePhi^{*}g_{E})_{\Theta_{1:T}}(\xi,\eta)
  &=g_{E}\bigl(D\Psi(D\Phi[\xi]),\,D\Psi(D\Phi[\eta])\bigr)\\
  &=\sum_{i,k}\sigma_{k}^{-1}[D\Phi[\xi]]_{i,k}\,\sigma_{k}^{-1}[D\Phi[\eta]]_{i,k}\\
  &=\sum_{i=1}^{T}\sum_{k=1}^{d}\sigma_{k}^{-2}\bigl[D\phi(\Theta_{i})[\xi_{i}]\bigr]_{k}\,
                                       \bigl[D\phi(\Theta_{i})[\eta_{i}]\bigr]_{k},
\end{aligned}
\]
where the last equality uses the block-diagonal form \eqref{eq:dphi-block}.
This is \eqref{eq:standardised-metric}.

Propositions~\ref{prop:traj-cfm}--\ref{prop:traj-rk} were proved using only
two facts: $\Phi$ is a global diffeomorphism, and the pullback metric is
$\Phi^{*}g_{E}$. Both of these hold for $\tildePhi$, with the pullback
metric now \eqref{eq:standardised-metric}. The identities
\eqref{eq:dphi-inverse} extend to $\tildePhi$ since they are general
chain-rule facts about any global diffeomorphism. Hence the proofs go through with $\Phi\to\tildePhi$.
\end{proof}

\begin{remark}[Why standardisation matters in practice]
Although $\tildePhi$ is geometrically equivalent to $\Phi$ under the rescaled
metric \eqref{eq:standardised-metric}, the choice of embedding matters in
practice because the CFM source distribution $\mathcal{N}(0,I)$ is defined
with respect to the unrescaled Euclidean inner product on $E$. The
standardisation \eqref{eq:Psi-def} aligns the source covariance with the
per-coordinate variance of the target, removing the under-dispersion observed
in the design-choice summary of Table~\ref{tab:ablation}.
\end{remark}

\subsection{Proposition~\ref{prop:forecast-cfm} (CFM equivalence under arbitrary couplings) and its proof}
\label{app:proof-forecast-cfm}

The forecasting variants of Section~\ref{sec:informedsources} replace the independent Gaussian
coupling, one by recentring the target and one by making the source depend on an extrapolation of the observed
history. The following
result shows the Riemannian/Euclidean equivalence survives any such coupling, so neither
construction weakens the guarantees of Appendix~\ref{app:pullback}.

\begin{proposition}[CFM equivalence under arbitrary couplings]
\label{prop:forecast-cfm}
Fix a history $z_{1:K}\in E^{K}$ and a label $y$. Let
$\pi(z^{f}_{0},z^{f}_{1}\mid z_{1:K},y)$ be \emph{any} joint probability measure on
$\Ef\times\Ef$ with finite second moment, and let
$u^{\Ef}_{\theta}(s,z^{f},z_{1:K},y)$ be the Euclidean field corresponding, via the pullback
identity
\begin{equation}
u^{\Ef}_{\theta}(s,z^{f},z_{1:K},y)
:=D\tildePhif\big(\tildePhif^{-1}(z^{f})\big)
\Big[u^{\Mtrajfut}_{\theta}\big(s,\tildePhif^{-1}(z^{f}),z_{1:K},y\big)\Big],
\label{eq:pullback-field-fut}
\end{equation}
to a Riemannian field $u^{\Mtrajfut}_{\theta}$ on
$(\Mtrajfut,\tildePhif^{*}g_{\Ef})$. Then the Riemannian CFM loss along the component-wise
log-Euclidean geodesic $\gamma^{f}$ between $\tildePhif^{-1}(z^{f}_{0})$ and
$\tildePhif^{-1}(z^{f}_{1})$ equals
\begin{equation}
\mathcal L^{\Mtrajfut}_{\mathrm{forecast}}(\theta)
=\mathbb E_{s,\,y,\,z_{1:K},\,(z^{f}_{0},z^{f}_{1})\sim\pi}
\norm{u^{\Ef}_{\theta}\big(s,(1-s)z^{f}_{0}+s\,z^{f}_{1},z_{1:K},y\big)
-\big(z^{f}_{1}-z^{f}_{0}\big)}^{2}_{\Ef}.
\label{eq:forecast-cfm-eucl}
\end{equation}
This holds in particular for the independent Gaussian coupling, for the mean-centred residual
coupling of \eqref{eq:classmean}--\eqref{eq:residual}, and for the random-walk random-walk
coupling $\pi=\qrw(z^{f}_{0}\mid z_{1:K})\,p_{\mathrm{data}}(z^{f}_{1}\mid z_{1:K},y)$ of
\eqref{eq:rwprior}.
\end{proposition}

\begin{proof}[Proof of Proposition~\ref{prop:forecast-cfm}]
The argument is the trajectory analogue of Proposition~\ref{prop:traj-cfm}, restricted to the
future block and conditioned on the history, and the point of it is that nothing anywhere
requires $z^{f}_{0}\perp z^{f}_{1}$.

\emph{Pointwise integrand identity.}
Fix any realisation $(s,y,z_{1:K},z^{f}_{0},z^{f}_{1})$ with $z^{f}_{0},z^{f}_{1}\in\Ef$, and
set
\begin{equation}
z^{f}_{s}:=(1-s)z^{f}_{0}+s\,z^{f}_{1},
\qquad
\gamma^{f}(s):=\tildePhif^{-1}(z^{f}_{s}).
\label{eq:fut-geodesic}
\end{equation}
Since $\tildePhif$ is a global diffeomorphism with pullback metric on $\Mtrajfut$
(Proposition~\ref{prop:standardisation} applied to the future block), $\gamma^{f}$ is the
geodesic on $(\Mtrajfut,\tildePhif^{*}g_{\Ef})$ between $\tildePhif^{-1}(z^{f}_{0})$ and
$\tildePhif^{-1}(z^{f}_{1})$, with velocity
$\dot\gamma^{f}(s)=D\tildePhif^{-1}(z^{f}_{s})[z^{f}_{1}-z^{f}_{0}]$. By exactly the
calculation used for Proposition~\ref{prop:traj-cfm}, applied to $\tildePhif$ and the field
\eqref{eq:pullback-field-fut},
\begin{equation}
\norm{u^{\Mtrajfut}_{\theta}(s,\gamma^{f}(s),z_{1:K},y)-\dot\gamma^{f}(s)}^{2}_{\gamma^{f}(s)}
=\norm{u^{\Ef}_{\theta}(s,z^{f}_{s},z_{1:K},y)-(z^{f}_{1}-z^{f}_{0})}^{2}_{\Ef}.
\label{eq:pointwise-fut}
\end{equation}

Equation \eqref{eq:pointwise-fut} holds for every fixed pair $(z^{f}_{0},z^{f}_{1})$, with no
probabilistic structure imposed on that pair. The derivation uses only the chain rule, the
isometry of $D\tildePhif$ between tangent spaces, and the algebraic identity
$D\tildePhif\circ D\tildePhif^{-1}=\mathrm{Id}_{\Ef}$. The conditioning variable $z_{1:K}$
enters only as an argument of the velocity field on both sides and does not touch the pullback
geometry.

Taking expectations of both sides of \eqref{eq:pointwise-fut} under $s\sim\mathcal U[0,1]$,
$y\sim\pi_{Y}$, $z_{1:K}\sim\pi_{\mathrm{hist}}$ and
$(z^{f}_{0},z^{f}_{1})\sim\pi(\cdot,\cdot\mid z_{1:K},y)$, the left-hand side becomes the
Riemannian CFM loss on $\Mtrajfut$ with source and target marginals
$(\tildePhif^{-1})_{\#}\pi_{0}$ and $(\tildePhif^{-1})_{\#}\pi_{1}$, while the right-hand side
is \eqref{eq:forecast-cfm-eucl}.
\end{proof}

\subsection{Proof of Proposition~\ref{prop:boundary} (boundary term as squared log-Euclidean distance)}
\label{app:proof-boundary}

The losses of Eqs.~\eqref{eq:bnd} and \eqref{eq:loss} are written in the
standardised embedding. We identify their counterparts in the
log-Euclidean geometry on $\Spd$ (Proposition~\ref{prop:boundary} for the boundary term,
Proposition~\ref{prop:temporal} for the temporal regulariser).

\begin{proposition}[Boundary term as squared log-Euclidean distance]
\label{prop:boundary}
Following the subscript convention of Section~\ref{sec:method} (first index: flow time;
second index: window index), let $\Theta_{K+1}:=\tildePhi^{-1}(z^{f}_{1,1})\in\Spd$ and
$\hatTheta_{K+1}:=\tildePhi^{-1}(\hat z^{f}_{1,1})\in\Spd$ be the decoded true and
predicted first-future precision matrices (window index $i=K+1$, flow time $s=1$). Then
\begin{equation}
\label{eq:bnd-as-le}
  \bigl\|\hat z^{f}_{1,1}-z^{f}_{1,1}\bigr\|_{2}^{2}
  \;=\;\sum_{k=1}^{d}\sigma_{k}^{-2}\bigl(\phi(\hatTheta_{K+1})_{k}-\phi(\Theta_{K+1})_{k}\bigr)^{2},
\end{equation}
which reduces, when $\sigma_{k}\equiv 1$, to the squared log-Euclidean
distance $\dLE^{2}(\hatTheta_{K+1},\Theta_{K+1})$.
\end{proposition}

\begin{proof}[Proof of Proposition~\ref{prop:boundary}]
The map $\Psi^{-1}$ is affine with diagonal action:
$\Psi^{-1}(\tilde z)_{i,k}-\Psi^{-1}(\tilde z')_{i,k}=\sigma_{k}(\tilde z_{i,k}-\tilde z'_{i,k})$
component-wise. Applied to $\tilde z=\hat z^{f}_{1,1}$ and $\tilde z'=z^{f}_{1,1}$
(both elements of $\R^{d}$, indexed by $k$),
\begin{equation}
\label{eq:psi-diff}
  \bigl(\Psi^{-1}(\hat z^{f}_{1,1})-\Psi^{-1}(z^{f}_{1,1})\bigr)_{k}
  \;=\;\sigma_{k}\bigl(\hat z^{f}_{1,1}-z^{f}_{1,1}\bigr)_{k}.
\end{equation}
The unstandardised differences are precisely the log-Euclidean residuals:
$\Psi^{-1}(\hat z^{f}_{1,1})=\phi(\hatTheta_{K+1})$ and
$\Psi^{-1}(z^{f}_{1,1})=\phi(\Theta_{K+1})$, by definition of $\tildePhi^{-1}$
and the fact that $\tildePhi^{-1}=\phi^{-1}\circ\Psi^{-1}$ applied
slice-by-slice. Therefore
\[
  \sigma_{k}\bigl(\hat z^{f}_{1,1}-z^{f}_{1,1}\bigr)_{k}
  \;=\;\phi(\hatTheta_{K+1})_{k}-\phi(\Theta_{K+1})_{k},
\]
and squaring and dividing by $\sigma_{k}^{2}$ gives the $k$-th term of
\eqref{eq:bnd-as-le}; summing over $k$ yields the identity. When
$\sigma_{k}\equiv 1$, the right-hand side equals
$\|\phi(\hatTheta_{K+1})-\phi(\Theta_{K+1})\|_{2}^{2}
=\|\log\hatTheta_{K+1}-\log\Theta_{K+1}\|_{F}^{2}
=\dLE^{2}(\hatTheta_{K+1},\Theta_{K+1})$
by \eqref{eq:veclt-isometry}.
\end{proof}

\subsection{Proof of Proposition~\ref{prop:temporal} (temporal regulariser as Riemannian smoothness)}
\label{app:proof-temporal}

The temporal regulariser \eqref{eq:loss} is stated in the standardised embedding;
Proposition~\ref{prop:temporal} identifies its geometric meaning on $\Spd$.

\begin{proposition}[Temporal regulariser as Riemannian smoothness]
\label{prop:temporal}
Let $\hatTheta_{i}:=\tildePhi^{-1}(\hat z_{1,i})\in\Spd$ for $i=1,\dots,T$ be
the decoded predicted-endpoint trajectory.
\begin{enumerate}
\item[(i)] With $\psi_{E}(\cdot)=\|\cdot\|_{2}^{2}$ (used in the smooth and
group-lasso TVGL settings),
\begin{equation}
\label{eq:temp-l2}
  \frac{1}{T-1}\sum_{i=2}^{T}\|\hat z_{1,i}-\hat z_{1,i-1}\|_{2}^{2}
  \;=\;\frac{1}{T-1}\sum_{i=2}^{T}\sum_{k=1}^{d}\sigma_{k}^{-2}
  \bigl(\phi(\hatTheta_{i})_{k}-\phi(\hatTheta_{i-1})_{k}\bigr)^{2},
\end{equation}
which is, up to per-coordinate scaling, the average squared log-Euclidean
distance between consecutive predicted matrices.
\item[(ii)] With $\psi_{E}(\cdot)=\|\cdot\|_{1}$ (used in the $\ell_{1}$ TVGL
setting),
\begin{equation}
\label{eq:temp-l1}
  \sum_{i=2}^{T}\|\hat z_{1,i}-\hat z_{1,i-1}\|_{1}
  \;=\;\sum_{i=2}^{T}\sum_{k=1}^{d}\sigma_{k}^{-1}
  \bigl|\phi(\hatTheta_{i})_{k}-\phi(\hatTheta_{i-1})_{k}\bigr|,
\end{equation}
i.e.\ a weighted $\ell_{1}$ norm of $\veclt$-vectorised
\emph{log-Euclidean} increments
$\log\hatTheta_{i}-\log\hatTheta_{i-1}$, not of raw precision increments
$\hatTheta_{i}-\hatTheta_{i-1}$.
\end{enumerate}
\end{proposition}

\begin{proof}[Proof of Proposition~\ref{prop:temporal}]
(i) By the diagonal action of $\Psi^{-1}$ (Eq.~\eqref{eq:psi-diff}), applied
slice-by-slice with $(\hat z_{1,i-1},\hat z_{1,i})$ in place of
$(z^{f}_{1,1},\hat z^{f}_{1,1})$,
\[
  \bigl\|\hat z_{1,i}-\hat z_{1,i-1}\bigr\|_{2}^{2}
  =\sum_{k}\sigma_{k}^{-2}\bigl(\phi(\hatTheta_{i})_{k}-\phi(\hatTheta_{i-1})_{k}\bigr)^{2}.
\]
Summing over $i=2,\dots,T$ and dividing by $T-1$ gives \eqref{eq:temp-l2}.

(ii) By the same diagonal action \eqref{eq:psi-diff},
$\sigma_{k}(\hat z_{1,i}-\hat z_{1,i-1})_{k}=\phi(\hatTheta_{i})_{k}-\phi(\hatTheta_{i-1})_{k}$
componentwise, so taking absolute values and dividing by $\sigma_{k}$,
$|\hat z_{1,i,k}-\hat z_{1,i-1,k}|=\sigma_{k}^{-1}|\phi(\hatTheta_{i})_{k}-\phi(\hatTheta_{i-1})_{k}|$,
and summing over $k$ and $i$ gives \eqref{eq:temp-l1}.
\end{proof}
\subsection{Summary}

Propositions~\ref{prop:prod-diffeo}--\ref{prop:boundary} together establish that the
entire \tvglcfm\ pipeline admits the same pullback equivalence as the
single-matrix case of \citet{collas2025}:
\begin{itemize}
\item Trajectory-level training, ODE integration, and explicit Runge--Kutta
sampling reduce to their Euclidean counterparts on $\R^{T\times d}$
(Props~\ref{prop:traj-cfm}, \ref{prop:traj-ode}, and~\ref{prop:traj-rk}).
\item Per-coordinate standardisation preserves the equivalence under a
rescaled product log-Euclidean metric
(Prop~\ref{prop:standardisation}).
\item History-conditioned forecasting with \emph{any} joint source--target coupling, including
both couplings of Section~\ref{sec:informedsources}, preserves the equivalence
(Prop~\ref{prop:forecast-cfm}).
\item The boundary term \eqref{eq:bnd} and the $\ell_{2}$ temporal
regulariser are, respectively, a squared log-Euclidean distance and the
average squared log-Euclidean increment between consecutive predicted
matrices; the $\ell_{1}$ temporal regulariser is a log-Euclidean equivalent
for TVGL's $\ell_{1}$ evolutionary penalty
(Props~\ref{prop:temporal} and~\ref{prop:boundary}).
\end{itemize}
Every decoded trajectory $\Theta_{1:T}=\tildePhi^{-1}(\tilde z)$ therefore
lies slice-wise in $\Spd$ by construction, without projection.

\section{Design choices}
Table~\ref{tab:ablation} summarises the main design choices and their effect. The embedding
standardisation is validated quantitatively by the generation and forecasting results above;
the remaining choices are motivated in Section~\ref{sec:method} and Appendix~\ref{sec:eval}.

\begin{table}[H]
\centering
\caption{Design choices and their effect.}
\label{tab:ablation}
\small
\begin{tabular}{p{0.34\textwidth}p{0.58\textwidth}}
\toprule
Choice & Effect \\
\midrule
Standardise embedding (\S\ref{sec:flow}) & Prevents under-dispersion; lowers discriminator AUC and corrects spectrum std-ratio. \\
OAS shrinkage covariance (\S\ref{sec:tvgl}) & Well-conditioned targets when $n\!\sim\!p$; stabilises TVGL and downstream metrics. \\
Scale-aware $\lambda,\beta$ (\S\ref{sec:tvgl-bg}) & Preserves class signal and genuine dynamics in the targets (avoids over-smoothing/over-sparsifying). \\
Covariance-space CAS (\S\ref{sec:eval}) & Measures the discriminative second-order structure; diagnoses where signal is lost. \\
Data-dependent couplings (\S\ref{sec:informedsources}) & The mean-centred residual coupling recentres the target on a training-set future mean and is stronger on EEG, The random-walk coupling instead replaces the source by an extrapolation of recent graph velocity. \\
Boundary term, forecasting (\S\ref{sec:forecast}) & Smooth, correct history to future handoff; avoids collapse to persistence and discontinuities. \\
Block (non-autoregressive) future (\S\ref{sec:forecast}) & Future windows attend jointly; avoids autoregressive error accumulation over the horizon. \\
\bottomrule
\end{tabular}
\end{table}

\section{Sensitivity to TVGL Regularization Parameters}
\label{app:tvgl_sensitivity}

TVGL introduces two regularization parameters that can significantly influence the target
precision trajectories used by TVGL--CFM: the sparsity parameter
$\lambda$ and the temporal regularization parameter $\beta$. Because
the proposed generative model is trained on TVGL-estimated trajectories,
it is important to verify that its performance is not tied to a narrow
choice of these parameters.

We therefore performed a controlled sensitivity study on
BNCI2014\_001 using the same pooled cross-session generation protocol as
in the main experiments. We varied
\[
    \lambda \in \{0.05,\,0.10,\,0.20\},
    \qquad
    \beta \in \{0.10,\,0.20,\,0.40\},
\]
which spans one-half to twice the values used in the main experiments,
$(\lambda,\beta)=(0.10,0.20)$. All other preprocessing, TVGL, model,
training, and evaluation settings were held fixed.

Table~\ref{tab:tvgl_lambda_beta_sensitivity} reports both the
discriminative quality of the corresponding real TVGL targets
(\emph{TVGL oracle}) and the quality of trajectories generated by
TVGL--CFM. Across the full $3\times3$ grid, the oracle itself is fairly stable:
its AUC varies only from 0.784 to 0.801 and its F1 score from 0.704 to
0.726. This indicates that the discriminative structure of the inferred
TVGL trajectories is preserved across a substantial neighborhood of the
default estimator configuration.

TVGL--CFM is similarly stable. Across all nine settings, generated
trajectories attain CAS AUC values between 0.684 and 0.736, F1 scores
between 0.616 and 0.679, and relative Fr\'echet distances between 1.041
and 1.133. Every setting yields an SPD fraction of 1.00. The default
configuration $(\lambda,\beta)=(0.10,0.20)$ attains the highest CAS AUC
in the grid, 0.736, while the neighboring setting
$(0.10,0.10)$ yields the highest F1 score, 0.679. The lowest relative
Fr\'echet distance, 1.041, occurs for $(0.20,0.10)$, although this
setting also exhibits the weakest CAS performance, suggesting a
trade-off between global distributional similarity and preservation of
class-discriminative structure.

Increasing temporal regularization to $\beta=0.40$ consistently raises
the relative Fr\'echet distance compared with the corresponding lower
$\beta$ settings, while stronger sparsification
($\lambda=0.20$) can reduce discriminative performance when combined
with weak temporal regularization. Nevertheless, the absence of a
sharp performance collapse across the grid indicates that TVGL--CFM
does not rely on a narrowly tuned TVGL configuration. These results
support the use of the moderate default
$(\lambda,\beta)=(0.10,0.20)$ employed in the main experiments.

\begin{table}[H]
\centering
\caption{Sensitivity of TVGL--CFM generation to the TVGL sparsity
parameter $\lambda$ and temporal regularization parameter $\beta$ on
BNCI2014\_001 under the pooled cross-session split. Oracle AUC/F1 are
computed from the corresponding real TVGL trajectories. CFM AUC/F1 are
CAS scores obtained from generated trajectories, and Rel-FD is the
relative Fr\'echet distance to the real held-out distribution. All
generated precision matrices are SPD. The default TVGL configuration
used in the main experiments is shown in bold.}
\label{tab:tvgl_lambda_beta_sensitivity}

\small
\setlength{\tabcolsep}{5pt}

\begin{tabular}{cccccccc}
\toprule
$\lambda$ &
$\beta$ &
Oracle AUC &
Oracle F1 &
CFM AUC &
CFM F1 &
Rel-FD &
SPD \\
\midrule
0.05 & 0.10 & 0.785 & 0.708 & 0.719 & 0.658 & 1.055 & 1.00 \\
0.05 & 0.20 & 0.787 & 0.712 & 0.726 & 0.656 & 1.073 & 1.00 \\
0.05 & 0.40 & \textbf{0.801} & 0.714 & 0.711 & 0.648 & 1.119 & 1.00 \\
0.10 & 0.10 & 0.784 & 0.704 & 0.735 & \textbf{0.679} & 1.056 & 1.00 \\
\textbf{0.10} & \textbf{0.20} &
0.790 & 0.710 & \textbf{0.736} & 0.668 & 1.066 & 1.00 \\
0.10 & 0.40 & 0.795 & \textbf{0.726} & 0.716 & 0.668 & 1.124 & 1.00 \\
0.20 & 0.10 & 0.797 & 0.714 & 0.684 & 0.616 & \textbf{1.041} & 1.00 \\
0.20 & 0.20 & 0.795 & 0.719 & 0.722 & 0.656 & 1.093 & 1.00 \\
0.20 & 0.40 & 0.796 & 0.717 & 0.715 & 0.653 & 1.133 & 1.00 \\
\bottomrule
\end{tabular}
\end{table}

\section{Robustness to the Temporal Graph Estimator}
\label{sec:loggle}

TVGL--CFM uses precision-matrix trajectories estimated by TVGL as its learning targets. Although TVGL provides a statistically principled and temporally coupled representation of dynamic networks, this raises the question of whether the downstream conditional flow-matching model depends specifically on the TVGL estimator. To test this directly, we replace TVGL with LOGGLE \citep{yang2020loggle}, a distinct estimator of time-varying Gaussian graphical models, while leaving the remainder of the trajectory-generation pipeline unchanged.

LOGGLE \citep{yang2020loggle} estimates a sequence of locally time-varying precision matrices by combining kernel-localised covariance estimation with group sparsity across a temporal neighbourhood. For a target time point $t_k$, let
\[
\mathcal{N}_{k,d}
=
\left\{
i:\left|t_i-t_k\right|\leq d
\right\},
\]
with $m_k = |\mathcal{N}_{k,d}|$. Given locally smoothed covariance estimates $\widehat{\Sigma}(t_i)$, the likelihood formulation of LOGGLE estimates the precision matrices within this neighbourhood by solving
\[
\min_{\{\Omega(t_i)\succ0\}_{i\in\mathcal{N}_{k,d}}}
\frac{1}{\sqrt{m_k}}
\sum_{i\in\mathcal{N}_{k,d}}
\left[
\operatorname{tr}
\left(
\Omega(t_i)\widehat{\Sigma}(t_i)
\right)
-
\log\det\Omega(t_i)
\right]
+
\lambda_{\mathrm{L}}
\sum_{u\neq v}
\sqrt{
\sum_{i\in\mathcal{N}_{k,d}}
\Omega_{uv}(t_i)^2
}.
\]

Unlike TVGL, which imposes an evolutionary penalty directly between adjacent precision matrices, LOGGLE constructs kernel-localised covariance estimates and encourages a common sparse edge structure within a local temporal neighbourhood. It therefore provides a substantially different mechanism for estimating dynamic precision trajectories.

We use the Epanechnikov kernel
\[
K_h(x)
=
\frac{3}{4}
\left(
1-\left(\frac{x}{h}\right)^2
\right)
\mathbb{I}(|x|\leq h),
\]
with normalised kernel weights for local covariance estimation. We implement the likelihood LOGGLE objective of \citet{yang2020loggle}, using their temporal-neighbourhood, Epanechnikov-kernel, and group-sparsity definitions. Optimisation is performed with a batched GPU ADMM solver; this changes only the numerical execution and not the underlying LOGGLE objective.

\paragraph{Experimental setup.}
We evaluate LOGGLE--CFM on BNCI2014\_001 under the same pooled cross-session generation protocol used in the main experiments. The training and held-out test splits each contain $1296$ EEG trials, with $p=22$ channels and $N=500$ temporal samples per trial. Each trial is represented by a trajectory of $T=20$ precision matrices, with graph locations placed at the centres of the same $20$ temporal bins used in the main pipeline.

For LOGGLE, we use the reference bandwidth rule
\[
h = 0.8N^{-1/5},
\]
which gives $h=0.230832$ for $N=500$, together with temporal neighbourhood width $d=0.2$ and sparsity parameter $\lambda_{\mathrm{L}}=0.25$. We fit on the correlation scale, globally centre each channel (\texttt{detrend=False}), and omit the optional post-estimation refitting stage (\texttt{refit=False}). The resulting LOGGLE precision matrices replace the TVGL targets directly; no additional inverse-covariance estimator is applied.

All downstream components are held fixed. Each LOGGLE precision matrix is mapped to the same log-Euclidean SPD representation used by TVGL--CFM, complete $20$-matrix trajectories are standardised using training-set statistics, and the same class-conditional trajectory flow model is trained in this representation for $1000$ epochs. Generation is evaluated using the same train-on-generated/test-on-real classifiability score (CAS), relative Fr\'echet distance, SPD validity, spectral diagnostics, and temporal-dynamics statistics used in the main generation experiments. Thus, the only substantive change in the representation pipeline is the temporal graph estimator that produces the real target trajectories.

\begin{table}[t]
\centering
\caption{
Robustness of conditional flow matching to replacement of TVGL with LOGGLE on BNCI2014\_001 under the pooled cross-session split.
The LOGGLE oracle reports classification using real LOGGLE trajectories, whereas LOGGLE--CFM reports CAS obtained by training on generated LOGGLE trajectories and evaluating on the held-out real LOGGLE test trajectories.
Rel-FD is the Fr\'echet distance normalised by the corresponding real train--test distance.
}
\label{tab:loggle}
\begin{tabular}{lcccc}
\toprule
Method &
CAS AUC $\uparrow$ &
CAS F1 $\uparrow$ &
Rel-FD $\downarrow$ &
SPD $\uparrow$ \\
\midrule
Real LOGGLE oracle & 0.803 & 0.737 & -- & 1.00 \\
LOGGLE--CFM        & 0.770 & 0.680 & 1.686 & 1.00 \\
\bottomrule
\end{tabular}
\end{table}

\paragraph{Results.}
Table~\ref{tab:loggle} shows that replacing TVGL with LOGGLE preserves strong class-discriminative information in the generated trajectories. The real LOGGLE trajectories attain an oracle AUC/F1 of $0.803/0.737$, while LOGGLE--CFM achieves a train-on-generated/test-on-real CAS AUC of $0.770$ and F1 of $0.680$. The AUC is therefore only $0.033$ below the corresponding real-trajectory oracle, despite changing the temporal graph estimator and learning the complete trajectory distribution generatively.

The structural and spectral diagnostics are also favourable. All generated precision matrices remain symmetric positive definite, yielding an SPD fraction of $1.00$. The pooled spectral Wasserstein distance between generated and real training trajectories is $0.0117$, while the generated-to-real spectral standard-deviation ratio is $1.002$, indicating close agreement in the distribution and spread of eigenvalues. Temporal variation is similarly preserved, with a temporal-dynamics ratio of $1.47$ for generated trajectories compared with $1.57$ for the real LOGGLE trajectories.

The absolute Fr\'echet distance between generated and held-out real LOGGLE trajectories is $1229.82$, compared with a real train--test reference distance of $729.42$, resulting in a relative Fr\'echet distance of $1.686$. This indicates that full high-dimensional distribution matching remains more difficult under the LOGGLE target representation, while class-discriminative utility, SPD validity, spectral structure, and temporal variation are nevertheless strongly preserved.

Overall, this experiment provides an estimator-level robustness check complementary to the TVGL $(\lambda,\beta)$ sensitivity analysis in Appendix~\ref{app:tvgl_sensitivity}. The sensitivity study tests robustness within the TVGL estimator family, whereas the LOGGLE experiment changes the temporal graph-learning mechanism itself. The resulting CAS AUC/F1 of $0.770/0.680$, full SPD validity, closely matched spectral statistics, and preserved temporal dynamics demonstrate that the conditional flow-matching formulation is not tied specifically to TVGL-estimated trajectories and can operate successfully on precision trajectories produced by a distinct time-varying graphical-model estimator.
\section{Complexity Analysis}
\label{app:complexity_scaling}

A limitation of the log-Euclidean representation is that its coordinate
dimension grows quadratically with the number of graph nodes. For a
$p\times p$ SPD precision matrix, the vectorized lower-triangular
log-Euclidean representation has dimension
\begin{equation}
    d(p)=\frac{p(p+1)}{2},
\end{equation}
so a trajectory of $T$ precision matrices contains $T\,d(p)$ Euclidean
coordinates. In addition to this representation growth, the dense matrix
eigendecompositions required by TVGL and by the matrix logarithm become
increasingly expensive as $p$ increases.

To quantify the practical scaling regime of the complete pipeline, we
performed a controlled synthetic benchmark using multivariate Gaussian
time series with increasing dimensionality. For each graph size
$p\in\{16,22,32,64,96,128,192,256,384,512\}$, we generated 128 training
and 64 test trajectories, each consisting of $T=20$ windows with 128
multivariate observations per window. The observations were passed
through the same processing chain used in the main experiments:
windowed OAS covariance estimation, TVGL inference, log-Euclidean
embedding, and TVGL--CFM training and sampling. We fixed the TVGL
parameters to $\lambda=0.1$, $\beta=0.2$, the group temporal penalty,
and 150 ADMM iterations. For the scaling benchmark, TVGL--CFM was
trained for 15 epochs; this experiment is intended to measure
computational cost rather than synthetic-data generation quality.

Table~\ref{tab:complexity_scaling} reports the resulting dimensionality,
runtime, and GPU-memory requirements. All tested graph sizes completed
successfully through $p=512$. At the largest setting, each precision
matrix is represented by $131{,}328$ log-Euclidean coordinates and the
complete 20-window trajectory contains $2{,}626{,}560$ coordinates.
Despite this representation size, TVGL--CFM training itself remained
moderate, requiring approximately $0.79$\,s per epoch and
$2.76$\,GiB of peak allocated GPU memory.

The dominant cost at larger graph sizes instead arises from the dense
SPD preprocessing stages. At $p=512$, TVGL inference required
$233.20$\,s and the log-Euclidean map required $991.75$\,s, compared
with $11.84$\,s for all 15 TVGL--CFM training epochs. The log-map stage
therefore becomes the main computational bottleneck as graph size
increases. Figure~\ref{fig:complexity_runtime} illustrates this
separation clearly: OAS covariance estimation and CFM sampling remain
small across the tested range, whereas TVGL and especially the matrix
logarithm increase rapidly with $p$.

Although the quadratic representation size makes very large
graphs increasingly expensive, the present implementation is not
restricted to the 13--22 node EEG networks considered in the main
experiments: the complete benchmark remained feasible through 512
nodes. The main practical limitation in this regime is the cost of
dense TVGL and matrix-function operations rather than flow-matching
training itself. Scaling to substantially larger graphs, particularly
graphs with thousands of nodes, would likely require sparse, low-rank,
block-structured, or approximate matrix-function representations.

\begin{table*}[t]
\centering
\caption{End-to-end computational scaling of TVGL--CFM on synthetic
multivariate Gaussian time series. $d=p(p+1)/2$ is the log-Euclidean
dimension of one precision matrix and $Td$ is the dimension of the
complete $T=20$ trajectory. CFM time/epoch is averaged over 15 training
epochs. Peak VRAM refers to peak allocated GPU memory during CFM
training. All configurations completed successfully.}
\label{tab:complexity_scaling}

\small
\setlength{\tabcolsep}{4pt}

\begin{tabular}{rrrrrrrr}
\toprule
$p$ &
$d$ &
$Td$ &
OAS (s) &
TVGL (s) &
Log map (s) &
CFM s/epoch &
Peak VRAM (GiB) \\
\midrule
16  & 136     & 2,720     & 0.323 & 1.864   & 0.146   & 0.178 & 0.135 \\
22  & 253     & 5,060     & 0.013 & 1.832   & 0.275   & 0.160 & 0.138 \\
32  & 528     & 10,560    & 0.014 & 1.807   & 0.670   & 0.100 & 0.143 \\
64  & 2,080   & 41,600    & 0.022 & 3.337   & 7.263   & 0.310 & 0.175 \\
96  & 4,656   & 93,120    & 0.033 & 5.275   & 13.365  & 0.323 & 0.220 \\
128 & 8,256   & 165,120   & 0.161 & 6.991   & 29.459  & 0.362 & 0.291 \\
192 & 18,528  & 370,560   & 0.125 & 27.192  & 78.456  & 0.361 & 0.494 \\
256 & 32,896  & 657,920   & 0.144 & 41.431  & 176.238 & 0.535 & 0.785 \\
384 & 73,920  & 1,478,400 & 0.167 & 104.774 & 416.382 & 0.713 & 1.606 \\
512 & 131,328 & 2,626,560 & 0.209 & 233.199 & 991.754 & 0.790 & 2.756 \\
\bottomrule
\end{tabular}
\end{table*}
\begin{figure}[H]
    \centering
    \includegraphics[width=0.88\textwidth]{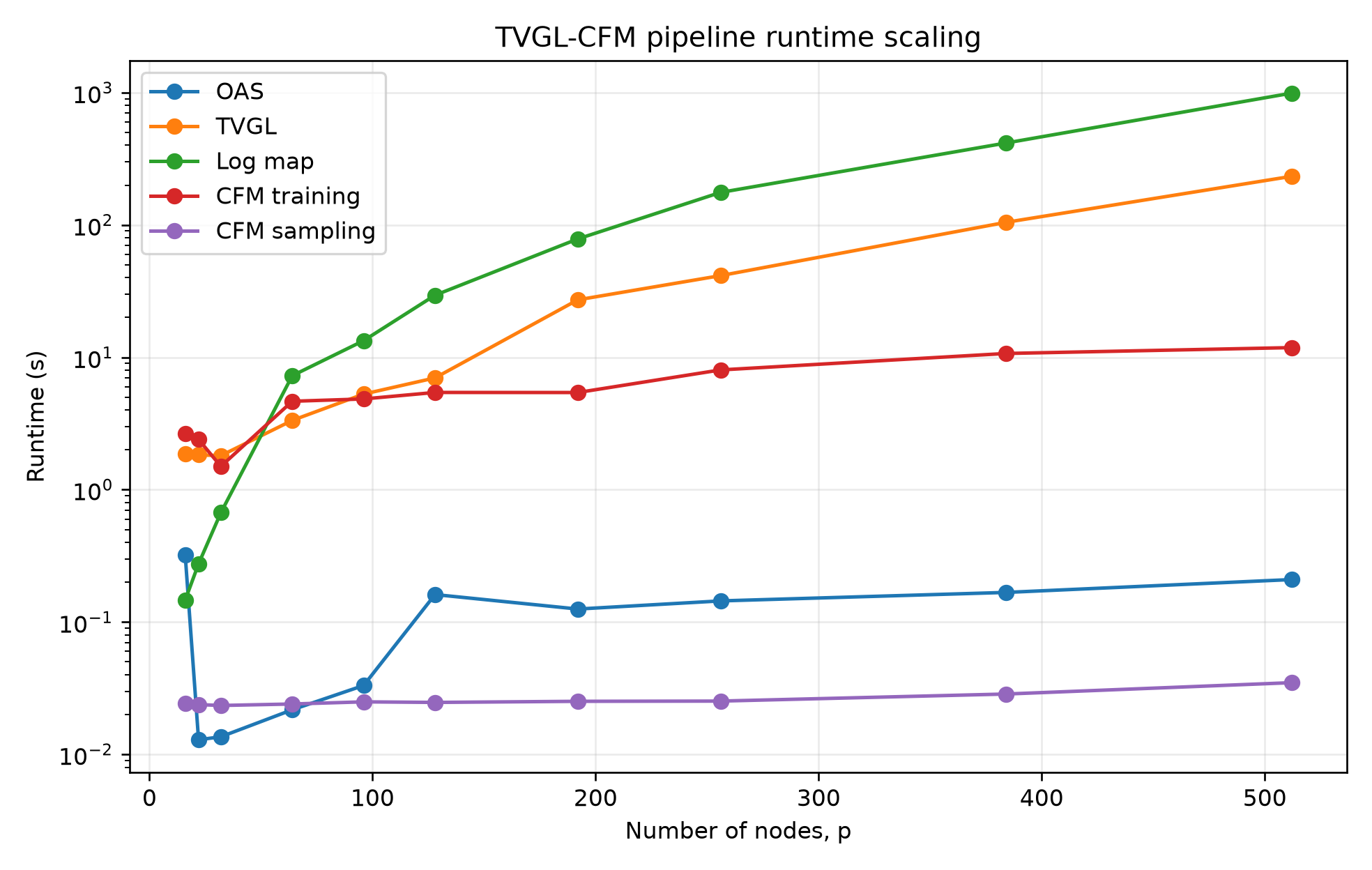}
    \caption{Runtime scaling of the complete TVGL--CFM pipeline with
    increasing graph size. The vertical axis is logarithmic. The
    log-Euclidean matrix logarithm becomes the dominant cost at larger
    $p$, followed by TVGL inference, while OAS covariance estimation,
    CFM training, and CFM sampling remain substantially smaller over the
    tested range.}
    \label{fig:complexity_runtime}
\end{figure}
\section{Experimental Details}
\label{app:experimental-details}

This appendix gives the full experimental configuration for the generation results
(Table~\ref{tab:eeg_generation}), the EEG forecasting results (Table~\ref{tab:eeg_forecasting}), and the
cross-domain forecasting results (Table~\ref{tab:forecast_main} and
Appendix~\ref{app:full-results}). All settings are
those used to produce the reported numbers; hyperparameters were fixed across datasets within
each experiment unless noted. Code will be released on acceptance.

\subsection{Datasets and preprocessing}
\label{app:datasets}

\begin{figure}[h]
    \centering
    \includegraphics[width=\linewidth]{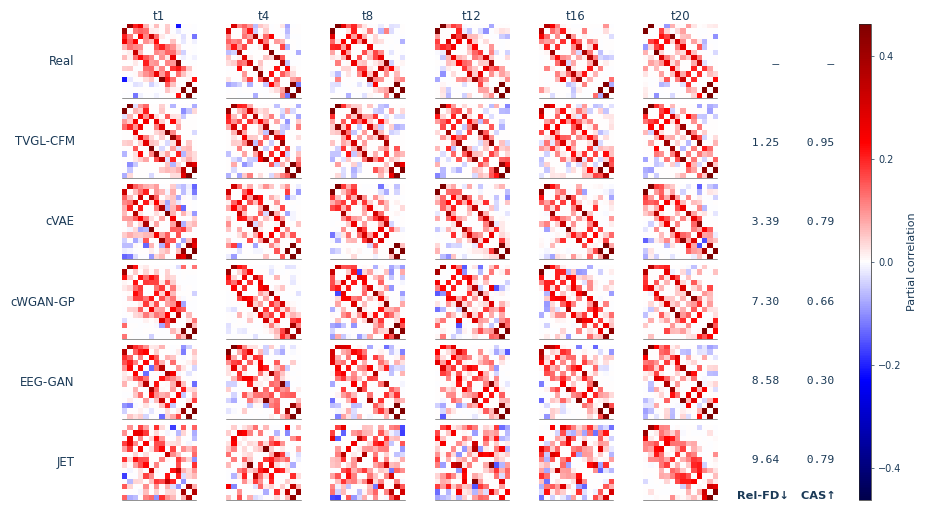}
    \caption{
    Qualitative comparison of TVGL precision-trajectory generation.
    Each row shows selected time windows from a generated or real TVGL trajectory,
    displayed as partial-correlation matrices. The TVGL-CFM row is visualized using
    a sparse post-processing readout for interpretability, while the reported
    Rel-FD and CAS values correspond to the unsparsified TVGL-CFM output used in
    quantitative evaluation. Lower Rel-FD indicates better distributional fidelity;
    higher CAS indicates better class-conditional utility.
    }
    \label{fig:tvgl_generation_qualitative}
\end{figure}

\paragraph{EEG motor imagery.}
We use four publicly available motor-imagery datasets from the MOABB benchmark
\citep{jayaram2018moabb}: BNCI2014\_001 ($p=22$ channels), BNCI2014\_002 ($p=15$),
BNCI2015\_001 ($p=13$), and Zhou2016 ($p=14$), each reduced to a two-class (right-hand vs.\
feet/left-hand) paradigm. Raw
recordings are band-pass filtered to $4$--$38$\,Hz and resampled to $128$\,Hz, EEG channels
are converted to $\mu$V, and each trial is split into $T=20$ contiguous windows. Per-channel
$z$-scoring is applied over time before windowing. We report a \emph{pooled cross-session}
split (a held-out session/run per subject) for generation and forecasting; means and standard
deviations are computed over held-out trials.

\paragraph{Cross-domain dynamical systems.}
Three simulated multivariate systems, exported as coarse-grained CSV trajectories, probe
domain generality: \textbf{Lorenz}, a network of diffusively coupled Lorenz oscillators
(deterministic chaos); \textbf{MacArthur}, a MacArthur consumer--resource competition system
(ecological population dynamics); and \textbf{Hopfield}, a continuous Hopfield
associative-memory network relaxing toward stored attractors. Each system is a single long
multivariate time series with no class label. We form overlapping pseudo-trials by sliding a
window of $\text{WIN\_SIZE}=15$ samples, using $T=K+h$ windows per trial with a fixed history
of $K=15$ windows and forecast horizons $h\in\{3,5,7\}$; the $h=5$ setting is the one
reported in the main table (Table~\ref{tab:forecast_main}) and the full sweep over
$h\in\{3,5,7\}$ is given in Appendix~\ref{app:full-results}. A temporal train/test split
($\text{train fraction}=0.7$) with a $\text{gap}=2$ windows between the training and test
segments prevents leakage across the boundary; node signals are normalised per coordinate
before covariance estimation.

\section{Evaluation}
\label{sec:eval}
We use a \emph{trajectory-respecting} protocol: each sample is the whole chain (for
forecasting, the future block), embedded per window and concatenated into a single
vector, so metrics are sensitive to temporal arrangement and not only to the marginal set
of matrices.

\paragraph{Distributional quality (both tasks).}
We report a relative Fr\'echet distance (Rel-FD) between generated and real trajectories in
the diffeomorphic embedding.

\paragraph{Classification-accuracy score, CAS (both tasks).}
Following the train-on-synthetic, test-on-real protocol \citep{ravuri2019}, we train a compact
multi-layer perceptron (MLP) probe on generated trajectories and evaluate it on real test data.
Rather than using logistic regression, we employ a small MLP to retain a lightweight and efficient
evaluation while allowing mild nonlinear decision boundaries. We also report a TVGL-precision
oracle, obtained by training the same probe on real training trajectories and evaluating it on
held-out real trajectories; the oracle-to-model gap therefore measures the discriminative information
lost during generation. By \eqref{eq:negation}, CAS features may equivalently be represented in
covariance space.

\paragraph{Forecasting-specific.}
We add (i) direct error: standardised-embedding RMSE and Frobenius RMSE on the future
precision matrices against the true future.
A natural baseline is \emph{persistence} (copy the last observed window across the
horizon), and for the mean-centred residual coupling we also report the \emph{training mean
alone}, i.e.\ the zero-residual baseline. A useful forecaster must
beat both on direct error while also matching the future distribution and dynamics.

\section{Implementation details and baselines}
\label{app:implementation}

\subsection{Embedding and TVGL--CFM architecture}
\label{app:arch}

Each precision trajectory is embedded window-wise through the log-Euclidean chart
(Eq.~\eqref{eq:trajectory-embedding}) and standardised per coordinate using training-set
statistics (Eq.~\eqref{eq:embedding-standardisation}). The velocity field $u_\theta$
(Eq.~\eqref{eq:ufield}) is a $4$-layer Transformer \emph{encoder} with model width
$d_{\mathrm{model}}=256$, $8$ attention heads, feed-forward width $4\,d_{\mathrm{model}}$, GELU
activations and no dropout. Flow time $s$, window time $\tau$, and the class label enter
through independent random-Fourier-feature embeddings ($N_{\mathrm{FF}}=12$ frequencies for
EEG, $8$ for the cross-domain systems) followed by an MLP, plus a learned class embedding.
Training uses AdamW \citep{loshchilov2019} (learning rate $5\times10^{-4}$, weight decay
$10^{-4}$), $500$ epochs and batch size $256$. The temporal-consistency weight is
$\lambda_{\mathrm{temp}}=0.1$. Sampling integrates the flow ODE with an explicit RK4 solver
($50$ steps for EEG, $30$ for the cross-domain systems). The forecasting models below use a
\emph{different} and smaller backbone; the two are not shared.

\paragraph{EEG forecasting architecture.}
The forecaster is a separate, smaller model: a $2$-layer Transformer \emph{encoder} for the
history (Eq.~\eqref{eq:histenc}) and a $2$-layer Transformer \emph{decoder} for the future
flow (Eq.~\eqref{eq:futurefield}), both with $d_{\mathrm{model}}=128$, $4$ attention heads,
feed-forward width $4\,d_{\mathrm{model}}$, GELU activations, norm-first blocks, no dropout,
and $N_{\mathrm{FF}}=12$. The history encoder returns $H\in\R^{K\times d_{\mathrm{model}}}$;
all $K=15$ ordered history tokens are retained. The future decoder conditions on $H$ by
cross-attention, so every future token can attend to every observed window, and future tokens
also self-attend to one another with no causal target mask. Both couplings share this backbone and differ only in the CFM source/target
construction.

\paragraph{Forecasting training and sampling.}
We train one forecaster per coupling. In the mean-centred residual variant
(Eqs.~\eqref{eq:classmean}--\eqref{eq:residual}) the source is a unit-scale Gaussian in
residual coordinates; the training mean, taken per class where labels exist, is estimated on
the fitting portion of the training split only, a held-out portion is used for checkpoint
selection. In the random-walk variant
(Eqs.~\eqref{eq:rwstats}--\eqref{eq:rwprior}) the source is built from $r=3$ recent history
increments with $\alpha_\mu=\alpha_\sigma=1$. In both cases the CFM term is computed in the
coupling's own coordinates while the structural losses are evaluated on the reconstructed
absolute future, with $\lambda_{\mathrm{temp}}=0.02$ and $\lambda_{\mathrm{bnd}}=0.10$. The
final EEG forecasting benchmark uses batch size $256$. Sampling uses $50$ RK4 steps and an
$M=32$-member ensemble; the point forecast is the ensemble mean in standardised log-Euclidean
coordinates, decoded once through $\phi^{-1}$, and the spread gives predictive uncertainty.

\subsection{Baselines}
\label{app:baselines}

We evaluate three families of baselines: \emph{non-learned forecasting controls},
\emph{direct TVGL-space models}, and \emph{raw-signal generators/forecasters}.

\paragraph{Non-learned forecasting controls.}
These baselines require no learned dynamics model.
\textsc{Training mean (no flow)} predicts $\mf$ of Eq.~\eqref{eq:classmean}
directly, i.e.\ the zero-residual baseline of Forecast
TVGL--CFM, using the same class-conditioned mean estimated on the same training
fitting subset. Because the mean is identical, the gap between this baseline and
Forecast TVGL--CFM isolates the contribution of the learned history-conditioned
residual transport. On the label-free simulated systems it is the global
training-set future mean.
\textsc{Persistence} copies the last observed precision matrix $\Theta_K$
across the full forecast horizon.
\textsc{Linear history-drift} extrapolates linearly in the standardised
log-Euclidean chart from the recent history.

\paragraph{Direct TVGL-space forecasting baselines.}
To isolate the effect of conditional flow matching from the benefit of modelling
precision trajectories directly, we compare against models trained on the same
standardised log-Euclidean TVGL representation as TVGL--CFM. Each forecaster
receives the same causal history
$\widetilde z_{1:K}\in\R^{K\times d}$ and predicts the same future block
$\widetilde z_{K+1:T}\in\R^{h\times d}$, with $d=p(p+1)/2$. Training
statistics are computed from the training split only, and all predictions are
decoded through the same inverse standardisation and log-Euclidean map.

Their outputs are predictions of
$\widetilde z_{K+1:T}$ itself rather than residuals relative to $\mf$.
Stochastic baselines use $M=32$ samples and form the point forecast by averaging
the generated absolute future trajectories in standardised log-Euclidean
coordinates before SPD decoding. All neural baselines are trained for
$500$ epochs.

\begin{itemize}
  \item \textbf{TVGL--cVAE-Forecast}: a conditional VAE that generates the
  complete future TVGL block jointly from the observed TVGL
  history and class label. During training, the encoder receives the observed
  history together with the true absolute future and parameterises a Gaussian
  latent variable. The decoder maps the observed history, class label, and
  latent sample directly to
  $\widehat z_{K+1:T}\in\R^{h\times d}$ in one pass. We use hidden width
  $512$, latent dimension $64$, KL weight $10^{-3}$, and learning rate
  $10^{-3}$.

  \item \textbf{TVGL--DDPM-Forecast}: a conditional diffusion model operating
  directly on the future TVGL block. Gaussian noise is added
  to the true future trajectory during training, and a Transformer denoiser
  conditioned on the complete observed history, class label, diffusion step,
  and future-window position learns the standard noise-prediction objective.
  At inference, reverse diffusion starts from Gaussian noise and directly
  generates a future trajectory in the standardised log-Euclidean
  chart. We use $50$ diffusion steps,
  a linear variance schedule from $10^{-4}$ to $2\times10^{-2}$, and learning
  rate $5\times10^{-4}$.

  \item \textbf{DyGraphformer-T}: a graph-native deterministic forecaster
  adapted to the TVGL representation. Each observed log-Euclidean TVGL vector
  is reconstructed as a symmetric matrix whose rows provide node-wise
  connectivity features. A recurrent node state and multi-head spatial
  attention model dynamic interactions across windows, after which temporal
  Transformer layers aggregate the observed trajectory. The prediction head
  directly outputs the complete \emph{absolute}
  $h\times d$ future TVGL block. We use node width $128$, $4$ spatial heads,
  $4$ temporal heads, and learning rate $5\times10^{-4}$.
\end{itemize}

\paragraph{Direct TVGL-space generation baselines.}
For generation, we use three additional baselines operating directly on complete
standardised log-Euclidean TVGL trajectories:

\begin{itemize}
  \item \textbf{TVGL--cVAE}: a class-conditional VAE that generates the complete
  absolute TVGL trajectory jointly from a Gaussian latent variable.

  \item \textbf{TVGL--DDPM}: a class-conditional diffusion model whose
  Transformer denoiser iteratively generates the complete absolute TVGL
  trajectory from noise.

  \item \textbf{Window-DIFFEO-CFM}: a conditional flow-matching model applied
  independently to each precision-matrix window. Its MLP velocity field is
  conditioned on the matrix representation, class, window time, and flow time,
  but there is no communication between windows. This tests whether the gain of
  TVGL--CFM comes from flow matching alone or from modelling the complete
  trajectory jointly.
\end{itemize}

\paragraph{Raw-signal generators and forecasters.}
Raw-signal baselines operate directly on multivariate EEG rather than on
estimated precision trajectories. Their generated or forecast signals are
subsequently passed through the \emph{same} OAS$+$TVGL estimation pipeline used
for the real data. This comparison therefore isolates direct TVGL-space
modelling from the alternative strategy of first predicting or generating raw
EEG and only then estimating its connectivity trajectory.

The predicted
future EEG is concatenated with the real observed EEG history, the resulting
completed trajectory is passed through the identical OAS$+$TVGL estimator, and
only the future precision-matrix windows are retained for evaluation. For
unconditional generation, the raw-signal models generate complete EEG
trajectories, with class conditioning where applicable. The final neural
comparisons use $500$ training epochs unless otherwise determined by the
corresponding source implementation.

Their configurations are:
\begin{itemize}
  \item \textbf{Basic-CNN}: a deterministic convolutional encoder--decoder that
  receives the observed raw EEG history and directly predicts the complete
   future EEG segment. The forecasting model uses hidden width
  $256$, learning rate $10^{-3}$, and $500$ epochs.

  \item \textbf{cVAE}: a conditional VAE operating on raw EEG, with latent
  dimension $64$, hidden width $512$, KL weight $10^{-3}$, and learning rate
  $10^{-3}$. For forecasting, the observed raw EEG history and class label
  condition generation of the complete  future EEG segment.
   For unconditional generation, the class label provides the
  conditioning information.

  \item \textbf{cWGAN-GP}: a conditional Wasserstein GAN with gradient penalty
  operating on raw EEG. For forecasting, its generator receives the observed
  raw EEG history, class label, and latent noise and directly generates the
future EEG segment. We use latent dimension $128$, hidden
  width $256$, $n_{\mathrm{critic}}=3$, gradient-penalty weight $10$, and
  learning rate $2\times10^{-4}$.

  \item \textbf{EEG-GAN}: an EEG-specific progressive convolutional
  Wasserstein GAN based on the EEG-GAN design, using interpolation-based
  upsampling, temporal convolutions, PixelNorm,
  minibatch-standard-deviation features, a one-sided gradient penalty, and
  critic drift regularisation.
  We use latent dimension $200$, $n_{\mathrm{critic}}=5$,
  gradient-penalty weight $10$, and drift coefficient
  $\varepsilon=10^{-3}$.

  \item \textbf{JET}: the official JET transformer flow-matching backbone for
  raw EEG \citep{jet}. We use the repository JET-B/16 configuration with patch
  length $200$, transformer width $256$, depth $6$, $8$ attention heads,
  exponential moving averages, and the source flow-matching objective. For
  forecasting, the observed EEG history is clamped throughout integration and
  JET generates  the future EEG segment. The resulting raw
  EEG forecast is then concatenated with the observed history and passed
  through OAS$+$TVGL.
\end{itemize}

For all raw-signal baselines, generated signals are clipped to $\pm5$
training-set standard deviations before inverse standardisation, covariance
estimation, and TVGL reconstruction.


\section{Seed variability for the EEG tables}
\label{app:eeg-std}
Tables~\ref{tab:eeg_generation} and~\ref{tab:eeg_forecasting} report means over three seeds
for legibility. The same numbers with their standard deviations are given below.

\begin{table*}[htbp]
\centering
\caption{Table~\ref{tab:eeg_generation} with standard deviations. Three-seed generation performance on EEG datasets, mean $\pm$ standard deviation across three random seeds. CAS AUC and F1 measure train-on-generated $\rightarrow$ test-on-real classifiability (higher is better), while Rel-FD is the Fr\'echet distance between generated and held-out real precision trajectories normalized by the real train--test reference (lower is better). Direct precision-space models generate TVGL trajectories directly; raw-EEG baselines generate signals that are subsequently mapped through the same OAS+TVGL pipeline.}
\label{tab:eeg_generation_std}

\renewcommand{\arraystretch}{2.15}
\setlength{\tabcolsep}{1.2pt}

\resizebox{\textwidth}{!}{
\begin{tabular}{lccccccccccccccc}
\toprule
& \multicolumn{3}{c}{Zhou2016}
& \multicolumn{3}{c}{BNCI2014\_001}
& \multicolumn{3}{c}{BNCI2014\_002}
& \multicolumn{3}{c}{BNCI2015\_001}
& \multicolumn{3}{c}{Mean} \\
\cmidrule(lr){2-4}
\cmidrule(lr){5-7}
\cmidrule(lr){8-10}
\cmidrule(lr){11-13}
\cmidrule(lr){14-16}
Method
& AUC & F1 & Rel-FD
& AUC & F1 & Rel-FD
& AUC & F1 & Rel-FD
& AUC & F1 & Rel-FD
& AUC & F1 & Rel-FD \\
\midrule

TVGL oracle
& 0.965 & 0.900 & --
& 0.794 & 0.722 & --
& 0.777 & 0.717 & --
& 0.762 & 0.692 & --
& -- & -- & -- \\

\midrule

\textbf{TVGL--CFM}
& \textbf{0.954$\pm$0.005}
& \textbf{0.893$\pm$0.008}
& \underline{\textit{1.101$\pm$0.093}}
& \textbf{0.718$\pm$0.023}
& \textbf{0.646$\pm$0.007}
& \textbf{1.191$\pm$0.096}
& \textbf{0.736$\pm$0.006}
& \textbf{0.678$\pm$0.021}
& \underline{\textit{1.022$\pm$0.020}}
& \textbf{0.717$\pm$0.015}
& \textbf{0.649$\pm$0.005}
& \textbf{1.061$\pm$0.099}
& \textbf{0.781$\pm$0.007}
& \textbf{0.716$\pm$0.008}
& \textbf{1.094$\pm$0.006} \\

TVGL--cVAE
& 0.879$\pm$0.016
& 0.799$\pm$0.009
& 1.111$\pm$0.002
& \underline{\textit{0.677$\pm$0.014}}
& \underline{\textit{0.620$\pm$0.015}}
& 1.938$\pm$0.000
& \underline{\textit{0.680$\pm$0.018}}
& \underline{\textit{0.621$\pm$0.016}}
& \textbf{0.994$\pm$0.007}
& \underline{\textit{0.693$\pm$0.009}}
& \underline{\textit{0.635$\pm$0.009}}
& 2.725$\pm$0.008
& \underline{\textit{0.732$\pm$0.006}}
& \underline{\textit{0.669$\pm$0.008}}
& 1.692$\pm$0.004 \\

TVGL--DDPM
& 0.889$\pm$0.017
& 0.792$\pm$0.016
& \textbf{1.075$\pm$0.006}
& 0.615$\pm$0.004
& 0.571$\pm$0.018
& \underline{\textit{1.572$\pm$0.008}}
& 0.624$\pm$0.066
& 0.587$\pm$0.054
& 1.098$\pm$0.015
& 0.603$\pm$0.034
& 0.569$\pm$0.019
& \underline{\textit{1.223$\pm$0.032}}
& 0.683$\pm$0.017
& 0.629$\pm$0.019
& \underline{\textit{1.242$\pm$0.005}} \\

Window-DIFFEO-CFM
& \underline{\textit{0.933$\pm$0.004}}
& \underline{\textit{0.853$\pm$0.011}}
& 1.208$\pm$0.008
& 0.626$\pm$0.010
& 0.596$\pm$0.012
& 2.145$\pm$0.007
& 0.661$\pm$0.004
& 0.621$\pm$0.017
& 1.386$\pm$0.007
& 0.664$\pm$0.007
& 0.593$\pm$0.007
& 1.536$\pm$0.013
& 0.721$\pm$0.005
& 0.666$\pm$0.007
& 1.569$\pm$0.002 \\

\midrule

JET $\rightarrow$ TVGL
& 0.573$\pm$0.067
& 0.398$\pm$0.110
& 9.216$\pm$0.326
& 0.508$\pm$0.050
& 0.385$\pm$0.081
& 11.374$\pm$0.288
& 0.506$\pm$0.015
& 0.356$\pm$0.041
& 7.845$\pm$0.191
& 0.505$\pm$0.010
& 0.334$\pm$0.001
& 31.738$\pm$0.328
& 0.523$\pm$0.026
& 0.368$\pm$0.044
& 15.043$\pm$0.260 \\

cVAE $\rightarrow$ TVGL
& 0.880$\pm$0.035
& 0.427$\pm$0.149
& 9.700$\pm$0.038
& 0.575$\pm$0.034
& 0.417$\pm$0.131
& 9.983$\pm$0.228
& 0.522$\pm$0.022
& 0.430$\pm$0.086
& 6.156$\pm$0.115
& 0.589$\pm$0.025
& 0.456$\pm$0.033
& 28.113$\pm$0.206
& 0.642$\pm$0.014
& 0.433$\pm$0.034
& 13.488$\pm$0.039 \\

cWGAN-GP $\rightarrow$ TVGL
& 0.678$\pm$0.042
& 0.570$\pm$0.030
& 2.602$\pm$0.099
& 0.520$\pm$0.037
& 0.513$\pm$0.026
& 3.498$\pm$0.213
& 0.507$\pm$0.035
& 0.500$\pm$0.018
& 2.018$\pm$0.234
& 0.510$\pm$0.018
& 0.477$\pm$0.050
& 8.523$\pm$0.177
& 0.554$\pm$0.008
& 0.515$\pm$0.027
& 4.160$\pm$0.075 \\

EEG-GAN $\rightarrow$ TVGL
& 0.375$\pm$0.088
& 0.335$\pm$0.003
& 8.288$\pm$0.458
& 0.537$\pm$0.052
& 0.373$\pm$0.068
& 12.802$\pm$0.561
& 0.547$\pm$0.051
& 0.391$\pm$0.056
& 4.474$\pm$0.255
& 0.566$\pm$0.053
& 0.417$\pm$0.066
& 5.792$\pm$0.127
& 0.507$\pm$0.044
& 0.379$\pm$0.015
& 7.839$\pm$0.167 \\

\bottomrule
\end{tabular}}
\end{table*}

\begin{table*}[htbp]
\centering
\caption{Table~\ref{tab:eeg_forecasting} with standard deviations. Three-seed forecasting performance on EEG datasets, mean $\pm$ standard deviation across three random seeds. Lower is better for AIRM and log-Euclidean RMSE (logE). Both TVGL--CFM couplings are reported: the mean-centred residual coupling and the random-walk coupling. The remaining learned forecasters predict the absolute future directly. Best in \textbf{bold}, second best \underline{\textit{underlined and italicized}}.}
\label{tab:eeg_forecasting_std}

\renewcommand{\arraystretch}{1.75}
\setlength{\tabcolsep}{2.5pt}

\resizebox{\textwidth}{!}{
\begin{tabular}{lcccccccccc}
\toprule
& \multicolumn{2}{c}{Zhou2016}
& \multicolumn{2}{c}{BNCI2014\_001}
& \multicolumn{2}{c}{BNCI2014\_002}
& \multicolumn{2}{c}{BNCI2015\_001}
& \multicolumn{2}{c}{Mean} \\
\cmidrule(lr){2-3}
\cmidrule(lr){4-5}
\cmidrule(lr){6-7}
\cmidrule(lr){8-9}
\cmidrule(lr){10-11}
Method
& AIRM & logE
& AIRM & logE
& AIRM & logE
& AIRM & logE
& AIRM & logE \\
\midrule

\textbf{TVGL--CFM (mean-residual)}
& \textbf{2.4154$\pm$0.0211}
& \textbf{2.3006$\pm$0.0231}
& 2.7476$\pm$0.0206
& 2.5449$\pm$0.0176
& \textbf{1.5527$\pm$0.0115}
& \textbf{1.4171$\pm$0.0049}
& \underline{\textit{1.9529$\pm$0.0085}}
& \underline{\textit{1.7937$\pm$0.0140}}
& \textbf{2.1671$\pm$0.0050}
& \textbf{2.0141$\pm$0.0059} \\

TVGL--CFM (random-walk)
& 2.9374$\pm$0.0341
& 2.7732$\pm$0.0340
& 3.4454$\pm$0.0205
& 3.1494$\pm$0.0228
& 1.8172$\pm$0.0094
& 1.6367$\pm$0.0061
& 2.1673$\pm$0.0297
& 1.9797$\pm$0.0186
& 2.5918$\pm$0.0083
& 2.3847$\pm$0.0063 \\

TVGL--cVAE-Forecast
& \underline{\textit{2.4523$\pm$0.0028}}
& \underline{\textit{2.3379$\pm$0.0017}}
& \underline{\textit{2.7067$\pm$0.0185}}
& \textbf{2.5031$\pm$0.0202}
& 1.6017$\pm$0.0039
& 1.4564$\pm$0.0044
& 1.9776$\pm$0.0084
& 1.8096$\pm$0.0081
& 2.1846$\pm$0.0049
& 2.0268$\pm$0.0051 \\

TVGL--DDPM-Forecast
& 2.7683$\pm$0.0080
& 2.6872$\pm$0.0093
& 3.3616$\pm$0.0015
& 3.2063$\pm$0.0009
& 1.7975$\pm$0.0021
& 1.6600$\pm$0.0019
& 2.1407$\pm$0.0027
& 2.0007$\pm$0.0032
& 2.5170$\pm$0.0025
& 2.3885$\pm$0.0028 \\

DyGraphformer-T
& 2.4606$\pm$0.0077
& 2.3487$\pm$0.0071
& \textbf{2.6960$\pm$0.0252}
& \underline{\textit{2.5074$\pm$0.0273}}
& \underline{\textit{1.5780$\pm$0.0112}}
& \underline{\textit{1.4418$\pm$0.0110}}
& \textbf{1.9404$\pm$0.0041}
& \textbf{1.7803$\pm$0.0044}
& \underline{\textit{2.1687$\pm$0.0076}}
& \underline{\textit{2.0196$\pm$0.0067}} \\

Training mean (no flow)
& 2.8473$\pm$0.0000
& 2.7690$\pm$0.0000
& 3.2869$\pm$0.0000
& 3.1430$\pm$0.0000
& 1.8397$\pm$0.0000
& 1.7075$\pm$0.0000
& 2.2037$\pm$0.0000
& 2.0674$\pm$0.0000
& 2.5444$\pm$0.0000
& 2.4217$\pm$0.0000 \\

Persistence
& 3.0118$\pm$0.0000
& 2.8219$\pm$0.0000
& 3.3624$\pm$0.0000
& 3.0462$\pm$0.0000
& 1.9314$\pm$0.0000
& 1.7206$\pm$0.0000
& 2.5205$\pm$0.0000
& 2.2729$\pm$0.0000
& 2.7065$\pm$0.0000
& 2.4654$\pm$0.0000 \\

\midrule

Basic-CNN $\rightarrow$ TVGL
& 9.8436$\pm$0.0057
& 9.8849$\pm$0.0058
& 10.8978$\pm$0.0078
& 10.9431$\pm$0.0078
& 5.9209$\pm$0.0004
& 5.9468$\pm$0.0004
& 8.7390$\pm$0.0002
& 8.7649$\pm$0.0002
& 8.8503$\pm$0.0026
& 8.8849$\pm$0.0026 \\

cVAE $\rightarrow$ TVGL
& 8.3927$\pm$0.0492
& 8.4227$\pm$0.0506
& 9.0986$\pm$0.0752
& 9.1200$\pm$0.0768
& 3.1431$\pm$0.0753
& 3.0970$\pm$0.0780
& 6.0469$\pm$0.0179
& 6.0299$\pm$0.0187
& 6.6703$\pm$0.0160
& 6.6674$\pm$0.0156 \\

cWGAN-GP $\rightarrow$ TVGL
& 9.8098$\pm$0.0313
& 9.8512$\pm$0.0313
& 10.7351$\pm$0.0226
& 10.7808$\pm$0.0226
& 4.9029$\pm$0.0487
& 4.9258$\pm$0.0475
& 5.3428$\pm$0.1482
& 5.2790$\pm$0.1606
& 7.6977$\pm$0.0590
& 7.7092$\pm$0.0614 \\

EEG-GAN $\rightarrow$ TVGL
& 4.3273$\pm$0.0068
& 4.1014$\pm$0.0255
& 5.1175$\pm$0.1299
& 4.7425$\pm$0.1271
& 2.9781$\pm$0.1515
& 2.7737$\pm$0.1410
& 3.1935$\pm$0.0302
& 2.9720$\pm$0.0315
& 3.9041$\pm$0.0232
& 3.6474$\pm$0.0224 \\

JET $\rightarrow$ TVGL
& 3.8611$\pm$0.0887
& 3.4902$\pm$0.0778
& 4.1264$\pm$0.2180
& 3.8915$\pm$0.2347
& 3.9463$\pm$0.0314
& 3.0594$\pm$0.0519
& 4.4094$\pm$0.3417
& 3.8906$\pm$0.4006
& 4.0858$\pm$0.0554
& 3.5829$\pm$0.0687 \\

\bottomrule
\end{tabular}}
\end{table*}

\section{Full chaotic-systems forecasting results}
\label{app:full-results}
Tables~\ref{tab:forecast_lorenz}--\ref{tab:forecast_hopfield} give the complete forecasting
results for the three chaotic systems of Section~\ref{sec:exp-chaotic}, for a fixed history of
$K=15$ windows and forecast horizons $h\in\{3,5,7\}$. Table~\ref{tab:forecast_main} in the
main paper averages over all nine system$\times$horizon settings.

\begin{table}[htbp]
\centering
\caption{\textbf{Lorenz forecasting results} for $K=15$ and
$h\in\{3,5,7\}$. Lower is better. Results are mean$\pm$std across
three training seeds.}
\label{tab:forecast_lorenz}

\small
\setlength{\tabcolsep}{3.5pt}
\renewcommand{\arraystretch}{1.08}

\begin{tabular}{llccc}
\toprule
\textbf{Method} & \textbf{Metric} & $h=3$ & $h=5$ & $h=7$ \\
\midrule

\multirow{2}{*}{TVGL--CFM (random-walk)}
& AIRM
& \textbf{1.418$\pm$0.036}
& \underline{\textit{1.465$\pm$0.036}}
& 1.512$\pm$0.037 \\
& RMSE
& \textbf{1.548$\pm$0.036}
& \underline{\textit{1.646$\pm$0.033}}
& 1.714$\pm$0.044 \\

\addlinespace[2pt]

\multirow{2}{*}{TVGL--CFM (train-mean residual)}
& AIRM
& 1.646$\pm$0.044
& 1.706$\pm$0.040
& 1.757$\pm$0.093 \\
& RMSE
& 1.816$\pm$0.040
& 1.929$\pm$0.051
& 2.007$\pm$0.101 \\

\addlinespace[2pt]

\multirow{2}{*}{TVGL--cVAE-Forecast}
& AIRM
& 1.558$\pm$0.021
& 1.568$\pm$0.015
& 1.593$\pm$0.008 \\
& RMSE
& 1.624$\pm$0.019
& 1.679$\pm$0.012
& 1.746$\pm$0.009 \\

\addlinespace[2pt]

\multirow{2}{*}{TVGL--DDPM-Forecast}
& AIRM
& 1.472$\pm$0.021
& \textbf{1.453$\pm$0.002}
& \textbf{1.449$\pm$0.007} \\
& RMSE
& 1.608$\pm$0.019
& 1.658$\pm$0.007
& \underline{\textit{1.681$\pm$0.010}} \\

\addlinespace[2pt]

\multirow{2}{*}{DyGraphformer-T}
& AIRM
& \underline{\textit{1.461$\pm$0.021}}
& 1.648$\pm$0.020
& 1.747$\pm$0.013 \\
& RMSE
& \underline{\textit{1.578$\pm$0.023}}
& 1.804$\pm$0.022
& 1.929$\pm$0.015 \\

\midrule

\multirow{2}{*}{Global train-mean future}
& AIRM
& 1.520$\pm$0.000
& 1.510$\pm$0.000
& \underline{\textit{1.500$\pm$0.000}} \\
& RMSE
& 1.594$\pm$0.000
& \textbf{1.626$\pm$0.000}
& \textbf{1.643$\pm$0.000} \\

\addlinespace[2pt]

\multirow{2}{*}{Mean train displacement}
& AIRM
& 1.991$\pm$0.000
& 2.005$\pm$0.000
& 2.003$\pm$0.000 \\
& RMSE
& 2.105$\pm$0.000
& 2.167$\pm$0.000
& 2.197$\pm$0.000 \\

\addlinespace[2pt]

\multirow{2}{*}{Persistence}
& AIRM
& 1.991$\pm$0.000
& 2.005$\pm$0.000
& 2.003$\pm$0.000 \\
& RMSE
& 2.105$\pm$0.000
& 2.166$\pm$0.000
& 2.196$\pm$0.000 \\

\addlinespace[2pt]

\multirow{2}{*}{Linear history-drift}
& AIRM
& 3.013$\pm$0.000
& 3.595$\pm$0.000
& 4.213$\pm$0.000 \\
& RMSE
& 3.152$\pm$0.000
& 3.810$\pm$0.000
& 4.501$\pm$0.000 \\

\midrule

\multirow{2}{*}{Basic-CNN $\rightarrow$ TVGL}
& AIRM
& 2.175$\pm$0.035
& 2.202$\pm$0.008
& 2.243$\pm$0.024 \\
& RMSE
& 2.282$\pm$0.044
& 2.337$\pm$0.008
& 2.392$\pm$0.029 \\

\addlinespace[2pt]

\multirow{2}{*}{LSTM $\rightarrow$ TVGL}
& AIRM
& 1.555$\pm$0.131
& 3.727$\pm$0.510
& 5.050$\pm$0.876 \\
& RMSE
& 2.059$\pm$0.184
& 4.483$\pm$0.495
& 5.536$\pm$0.689 \\

\addlinespace[2pt]

\multirow{2}{*}{cVAE $\rightarrow$ TVGL}
& AIRM
& 3.640$\pm$0.095
& 3.658$\pm$0.082
& 3.629$\pm$0.085 \\
& RMSE
& 3.761$\pm$0.096
& 3.804$\pm$0.084
& 3.787$\pm$0.088 \\

\addlinespace[2pt]

\multirow{2}{*}{cWGAN-GP $\rightarrow$ TVGL}
& AIRM
& 2.231$\pm$0.190
& 2.256$\pm$0.235
& 2.287$\pm$0.260 \\
& RMSE
& 2.326$\pm$0.214
& 2.415$\pm$0.250
& 2.465$\pm$0.334 \\

\bottomrule
\end{tabular}
\end{table}

\begin{table}[htbp]
\centering
\caption{\textbf{MacArthur forecasting results} for $K=15$ and
$h\in\{3,5,7\}$. Lower is better.}
\label{tab:forecast_macarthur}

\small
\setlength{\tabcolsep}{3.5pt}
\renewcommand{\arraystretch}{1.08}

\begin{tabular}{llccc}
\toprule
\textbf{Method} & \textbf{Metric} & $h=3$ & $h=5$ & $h=7$ \\
\midrule

\multirow{2}{*}{TVGL--CFM (random-walk)}
& AIRM
& 0.880$\pm$0.007
& 1.047$\pm$0.004
& 1.077$\pm$0.018 \\
& RMSE
& 0.908$\pm$0.009
& 1.075$\pm$0.004
& 1.101$\pm$0.017 \\

\addlinespace[2pt]

\multirow{2}{*}{TVGL--CFM (train-mean residual)}
& AIRM
& 0.894$\pm$0.021
& 1.087$\pm$0.025
& 1.122$\pm$0.003 \\
& RMSE
& 0.930$\pm$0.021
& 1.112$\pm$0.019
& 1.144$\pm$0.001 \\

\addlinespace[2pt]

\multirow{2}{*}{TVGL--cVAE-Forecast}
& AIRM
& 0.982$\pm$0.008
& \underline{\textit{1.028$\pm$0.004}}
& \underline{\textit{1.058$\pm$0.002}} \\
& RMSE
& 0.994$\pm$0.007
& \textbf{1.045$\pm$0.004}
& \underline{\textit{1.078$\pm$0.003}} \\

\addlinespace[2pt]

\multirow{2}{*}{TVGL--DDPM-Forecast}
& AIRM
& 1.075$\pm$0.002
& 1.071$\pm$0.001
& 1.075$\pm$0.002 \\
& RMSE
& 1.085$\pm$0.002
& 1.087$\pm$0.001
& 1.093$\pm$0.001 \\

\addlinespace[2pt]

\multirow{2}{*}{DyGraphformer-T}
& AIRM
& \underline{\textit{0.879$\pm$0.009}}
& \textbf{1.019$\pm$0.001}
& 1.117$\pm$0.010 \\
& RMSE
& \underline{\textit{0.904$\pm$0.011}}
& \underline{\textit{1.051$\pm$0.002}}
& 1.150$\pm$0.010 \\

\midrule

\multirow{2}{*}{Global train-mean future}
& AIRM
& 1.058$\pm$0.000
& 1.055$\pm$0.000
& \textbf{1.053$\pm$0.000} \\
& RMSE
& 1.066$\pm$0.000
& 1.067$\pm$0.000
& \textbf{1.067$\pm$0.000} \\

\addlinespace[2pt]

\multirow{2}{*}{Mean train displacement}
& AIRM
& 1.568$\pm$0.000
& 1.539$\pm$0.000
& 1.513$\pm$0.000 \\
& RMSE
& 1.572$\pm$0.000
& 1.558$\pm$0.000
& 1.539$\pm$0.000 \\

\addlinespace[2pt]

\multirow{2}{*}{Persistence}
& AIRM
& 1.568$\pm$0.000
& 1.539$\pm$0.000
& 1.513$\pm$0.000 \\
& RMSE
& 1.572$\pm$0.000
& 1.558$\pm$0.000
& 1.539$\pm$0.000 \\

\addlinespace[2pt]

\multirow{2}{*}{Linear history-drift}
& AIRM
& 2.244$\pm$0.000
& 2.644$\pm$0.000
& 3.051$\pm$0.000 \\
& RMSE
& 2.242$\pm$0.000
& 2.680$\pm$0.000
& 3.126$\pm$0.000 \\

\midrule

\multirow{2}{*}{Basic-CNN $\rightarrow$ TVGL}
& AIRM
& 2.310$\pm$0.029
& 2.406$\pm$0.054
& 2.465$\pm$0.023 \\
& RMSE
& 2.386$\pm$0.039
& 2.506$\pm$0.055
& 2.565$\pm$0.027 \\

\addlinespace[2pt]

\multirow{2}{*}{LSTM $\rightarrow$ TVGL}
& AIRM
& \textbf{0.706$\pm$0.051}
& 1.456$\pm$0.093
& 2.214$\pm$0.140 \\
& RMSE
& \textbf{0.791$\pm$0.059}
& 1.695$\pm$0.110
& 2.622$\pm$0.179 \\

\addlinespace[2pt]

\multirow{2}{*}{cVAE $\rightarrow$ TVGL}
& AIRM
& 4.777$\pm$0.029
& 5.259$\pm$0.016
& 5.434$\pm$0.014 \\
& RMSE
& 4.851$\pm$0.024
& 5.344$\pm$0.014
& 5.522$\pm$0.014 \\

\addlinespace[2pt]

\multirow{2}{*}{cWGAN-GP $\rightarrow$ TVGL}
& AIRM
& 2.501$\pm$0.273
& 2.264$\pm$0.193
& 2.381$\pm$0.103 \\
& RMSE
& 2.558$\pm$0.300
& 2.344$\pm$0.201
& 2.505$\pm$0.160 \\

\bottomrule
\end{tabular}
\end{table}

\begin{table}[htbp]
\centering
\caption{\textbf{Hopfield forecasting results} for $K=15$ and
$h\in\{3,5,7\}$. Lower is better.}
\label{tab:forecast_hopfield}

\small
\setlength{\tabcolsep}{3.5pt}
\renewcommand{\arraystretch}{1.08}

\begin{tabular}{llccc}
\toprule
\textbf{Method} & \textbf{Metric} & $h=3$ & $h=5$ & $h=7$ \\
\midrule

\multirow{2}{*}{TVGL--CFM (random-walk)}
& AIRM
& \textbf{0.626$\pm$0.002}
& \textbf{0.759$\pm$0.015}
& \underline{\textit{0.832$\pm$0.005}} \\
& RMSE
& \textbf{0.706$\pm$0.006}
& \textbf{0.848$\pm$0.016}
& \underline{\textit{0.911$\pm$0.006}} \\

\addlinespace[2pt]

\multirow{2}{*}{TVGL--CFM (train-mean residual)}
& AIRM
& \underline{\textit{0.636$\pm$0.026}}
& 0.790$\pm$0.034
& 0.858$\pm$0.015 \\
& RMSE
& \underline{\textit{0.724$\pm$0.026}}
& 0.907$\pm$0.039
& 0.983$\pm$0.018 \\

\addlinespace[2pt]

\multirow{2}{*}{TVGL--cVAE-Forecast}
& AIRM
& 0.747$\pm$0.004
& 0.801$\pm$0.003
& \textbf{0.830$\pm$0.002} \\
& RMSE
& 0.794$\pm$0.005
& \underline{\textit{0.851$\pm$0.001}}
& \textbf{0.883$\pm$0.002} \\

\addlinespace[2pt]

\multirow{2}{*}{TVGL--DDPM-Forecast}
& AIRM
& 0.845$\pm$0.021
& 0.890$\pm$0.015
& 0.908$\pm$0.014 \\
& RMSE
& 0.919$\pm$0.017
& 0.961$\pm$0.015
& 0.982$\pm$0.012 \\

\addlinespace[2pt]

\multirow{2}{*}{DyGraphformer-T}
& AIRM
& 0.667$\pm$0.006
& \underline{\textit{0.784$\pm$0.008}}
& 0.850$\pm$0.015 \\
& RMSE
& 0.727$\pm$0.006
& 0.857$\pm$0.009
& 0.931$\pm$0.012 \\

\midrule

\multirow{2}{*}{Global train-mean future}
& AIRM
& 0.905$\pm$0.000
& 0.900$\pm$0.000
& 0.899$\pm$0.000 \\
& RMSE
& 0.935$\pm$0.000
& 0.937$\pm$0.000
& 0.939$\pm$0.000 \\

\addlinespace[2pt]

\multirow{2}{*}{Mean train displacement}
& AIRM
& 1.289$\pm$0.000
& 1.289$\pm$0.000
& 1.270$\pm$0.000 \\
& RMSE
& 1.342$\pm$0.000
& 1.351$\pm$0.000
& 1.339$\pm$0.000 \\

\addlinespace[2pt]

\multirow{2}{*}{Persistence}
& AIRM
& 1.289$\pm$0.000
& 1.288$\pm$0.000
& 1.270$\pm$0.000 \\
& RMSE
& 1.342$\pm$0.000
& 1.350$\pm$0.000
& 1.339$\pm$0.000 \\

\addlinespace[2pt]

\multirow{2}{*}{Linear history-drift}
& AIRM
& 1.796$\pm$0.000
& 2.110$\pm$0.000
& 2.409$\pm$0.000 \\
& RMSE
& 1.827$\pm$0.000
& 2.170$\pm$0.000
& 2.495$\pm$0.000 \\

\midrule

\multirow{2}{*}{Basic-CNN $\rightarrow$ TVGL}
& AIRM
& 1.840$\pm$0.029
& 2.012$\pm$0.021
& 2.082$\pm$0.043 \\
& RMSE
& 1.977$\pm$0.032
& 2.186$\pm$0.026
& 2.269$\pm$0.064 \\

\addlinespace[2pt]

\multirow{2}{*}{LSTM $\rightarrow$ TVGL}
& AIRM
& 0.790$\pm$0.036
& 1.041$\pm$0.131
& 1.626$\pm$0.052 \\
& RMSE
& 0.903$\pm$0.039
& 1.204$\pm$0.168
& 1.866$\pm$0.065 \\

\addlinespace[2pt]

\multirow{2}{*}{cVAE $\rightarrow$ TVGL}
& AIRM
& 2.912$\pm$0.122
& 3.356$\pm$0.084
& 3.649$\pm$0.060 \\
& RMSE
& 3.052$\pm$0.120
& 3.529$\pm$0.076
& 3.839$\pm$0.058 \\

\addlinespace[2pt]

\multirow{2}{*}{cWGAN-GP $\rightarrow$ TVGL}
& AIRM
& 1.690$\pm$0.409
& 1.851$\pm$0.633
& 1.786$\pm$0.365 \\
& RMSE
& 1.769$\pm$0.364
& 2.064$\pm$0.925
& 1.978$\pm$0.538 \\

\bottomrule
\end{tabular}
\end{table}

\section{Compute}
\label{app:compute}

All experiments were run on a single NVIDIA L40 GPU. TVGL target estimation was batched across
sequences and windows per dataset. Each TVGL--CFM
generator or forecaster trained in under one hour per dataset, while sampling a full forecast
ensemble required only a few seconds per batch.

\section{Use of Large Language Models}
\label{app:llm}

Large language models were used to assist with notation refinement, formatting, proofreading,
clarity of presentation, figure design, and preliminary proof checking. The authors reviewed,
verified, and, where necessary, revised all generated suggestions before inclusion in the manuscript.
All scientific claims, mathematical arguments, experimental choices, implementations, and reported
results remain the responsibility of the authors.

\end{document}